\documentclass{article}

\usepackage[preprint,nonatbib]{neurips_2024}

\usepackage[utf8]{inputenc} 
\usepackage[T1]{fontenc}    
\usepackage{hyperref}       
\usepackage{url}            
\usepackage{booktabs}       
\usepackage{amsfonts}       
\usepackage{nicefrac}       
\usepackage{microtype}      
\usepackage{xcolor}         
\usepackage{cite}
\usepackage{amsmath,amssymb,amsthm}
\usepackage{mathtools}
\usepackage{subcaption}
\usepackage{algorithm,algorithmic}
\usepackage{wrapfig}
\usepackage{tablefootnote}

\DeclareMathOperator{\tr}{tr}
\DeclareMathOperator{\Var}{Var}
\DeclareMathOperator{\sign}{sign}
\DeclareMathOperator{\Lap}{Lap}

\newtheorem{thm}{Theorem}
\newtheorem{ass}{Assumption}
\newtheorem{defi}{Definition}
\newtheorem{lem}{Lemma}
\newtheorem{rmk}{Remark}
\newtheorem{prop}{Proposition}

\newcommand{\norm}[1]{\left\lVert#1\right\rVert}

\title{Learning with User-Level Local Differential Privacy}

\author{%
	Puning Zhao$^1$ \quad Li Shen$^2$ \quad Rongfei Fan$^3$ Qingming Li$^4$ \quad Huiwen Wu$^1$ \quad Jiafei Wu$^1$ \quad Zhe Liu$^1$ \\
	$^1$ Zhejiang Lab \quad $^2$ Sun Yat-Sen University \quad $^3$ Beijing Institute of Technology \quad $^4$ Zhejiang University\\
	\texttt{\{pnzhao,wujiafei,zhe.liu\}@zhejianglab.com},\\
	\texttt{mathshenli@gmail.com}, \texttt{fanrongfei@bit.edu.cn},
	\texttt{liqm@zju.edu.cn},
	\texttt{huiwen0820@outlook.com}
}

\begin{document}

\maketitle

\begin{abstract}
User-level privacy is important in distributed systems. Previous research primarily focuses on the central model, while the local models have received much less attention. Under the central model, user-level DP is strictly stronger than the item-level one. However, under the local model, the relationship between user-level and item-level LDP becomes more complex, thus the analysis is crucially different. In this paper, we first analyze the mean estimation problem and then apply it to stochastic optimization, classification, and regression. In particular, we propose adaptive strategies to achieve optimal performance at all privacy levels. Moreover, we also obtain information-theoretic lower bounds, which show that the proposed methods are minimax optimal up to logarithmic factors. Unlike the central DP model, where user-level DP always leads to slower convergence, our result shows that under the local model, the convergence rates are nearly the same between user-level and item-level cases for distributions with bounded support. For heavy-tailed distributions, the user-level rate is even faster than the item-level one.


\end{abstract}



\section{Introduction}

Differential privacy (DP) \cite{dwork2006calibrating} is one of the mainstream schemes for privacy protection. The traditional DP framework is item-level, which focuses on the privacy of each sample \cite{dwork2014algorithmic}. However, in many real-world scenarios such as federated learning \cite{kairouz2021advances,geyer2017differentially,mcmahan2018learning,wang2019beyond,wei2020federated,wei2021user,huang2023federated,fu2024differentially}, each user provides multiple samples, which need to be treated as a whole for privacy protection. Therefore, in recent years, user-level differential privacy has emerged and has received widespread attention from researchers \cite{liu2020learning,levy2021learning,ghazi2021user,ghazi2023user,liu2024user,acharya2023discrete}.

Existing research on user-level DP mainly focuses on central models \cite{liu2020learning,levy2021learning,ghazi2021user,ghazi2023user}. Much less effort has been made on local models. An exception is \cite{acharya2023discrete}, which analyzes the discrete distribution estimation problem under user-level local differential privacy (LDP). Nevertheless, user-level LDP has a much wider range of potential applications \cite{yang2020local}. Other statistical problems have been rarely explored. A challenge under the local model is to make the method suitable for all privacy levels. Under the central model, we can just let the noise scales as $1/\epsilon$, while the algorithm structure remains the same for different $\epsilon$. Nevertheless, under the local model, privatization takes place before aggregation. To achieve optimal performance, privacy mechanisms need to be tailored to each privacy level $\epsilon$. Moreover, it is also not straightforward to derive the information-theoretic lower bound, since under the local model, user-level LDP is not necessarily stronger than the item-level one, thus the item-level lower bounds do not directly lead to user-level counterparts. 

In this paper, we discuss a wide range of statistical tasks under user-level $\epsilon$-LDP. We analyze the mean estimation problem first, including one-dimensional and multi-dimensional cases. We then apply the mean estimation methods to other tasks, including stochastic optimization, classification, and regression. For each task, we provide algorithms and analyze the theoretical convergence rates. Moreover, we derive the information-theoretic lower bounds, which shows that the newly proposed methods are minimax rate optimal up to a logarithm factor. The results are shown in Table \ref{tab:comp}, in which the non-private term is omitted for simplicity. Under central DP, user-level DP is strictly stronger than item-level one, and thus always leads to a slower convergence rate \cite{levy2021learning}. On the contrary, under the local model, the same convergence rates are derived between user-level and item-level cases for distributions with bounded support. If the distribution is heavy-tailed, then perhaps surprisingly, the user-level rate is even faster, such as those shown in the second row in Table \ref{tab:comp}.

To achieve optimal performance at all privacy levels, we design an adaptive method tailored to each $\epsilon$. For example, for the $d$ dimensional mean estimation problem, with $\epsilon<1$, the dataset is divided into $d$ groups to estimate each component. The number of groups then decreases with increasing $\epsilon$. With very large $\epsilon$, there is only one group, so the error becomes close to the non-private case. With the algorithm varying among different $\epsilon$, it is natural to see phase transitions in the bounds in Table \ref{tab:comp}. For the establishment of minimax lower bounds, we revisit classical minimax theory \cite{tsybakov2009introduction} to derive tight bounds on the distances between privatized samples.

\begin{table}[t]
\small 
	\begin{center}
 	\caption{\small Comparison of performance under user-level and item-level LDP.}\label{tab:comp}
	\begin{tabular}{ccc}
	\hline
	Tasks & user-level & item-level\\
	& $n$ users, $m$ samples per user & $nm$ samples\\
	\hline
	Mean, bounded &$\tilde{O}\left(\frac{d}{nm(\epsilon^2\wedge \epsilon)}\right)$&$O\left(\frac{d}{nm(\epsilon^2\wedge \epsilon)}\right)$ \cite{feldman2021lossless,chen2020breaking,asi2022optimal}\\	
	Mean, heavy-tail &$\tilde{O}\left(\frac{d\ln m}{mn(\epsilon^2\wedge \epsilon)}\vee \left(\frac{d}{m^2n(\epsilon^2\wedge \epsilon)}\right)^{1-\frac{1}{p}}\right)$ & ${O}\left(\left(\frac{d}{nm(\epsilon^2\wedge \epsilon)}\right)^{1-\frac{1}{p}}\right)$ \cite{duchi2018minimax} \tablefootnote{For heavy-tailed distribution, \cite{duchi2018minimax} analyzed the one dimensional case. We generalize it to $d$ dimensions.}\\
    Stochastic optimization &$\tilde{O}\left(\sqrt{\frac{d}{nm(\epsilon^2\wedge \epsilon)}}\right)$ & $\tilde{O}\left(\sqrt{\frac{d}{nm(\epsilon^2\wedge \epsilon)}}\right)$ \cite{duchi2013local} \tablefootnote{\cite{duchi2013local} analyzed the case with $\epsilon\leq 1/4$. We generalize it to larger $\epsilon$.} \\
	Classification & $\tilde{O}\left((mn(\epsilon^2\wedge \epsilon))^{-\frac{\beta(1+\gamma)}{2(d+\beta)}}\right)$ &$O\left((mn(\epsilon^2\wedge \epsilon))^{-\frac{\beta(1+\gamma)}{2(d+\beta)}}\right)$ \cite{berrett2019classification}\\
	Regression &$\tilde{O}\left((mn(\epsilon^2\wedge \epsilon))^{-\frac{\beta}{d+\beta}} \right)$&$O\left((mn(\epsilon^2\wedge \epsilon))^{-\frac{\beta}{d+\beta}} \right)$ \cite{berrett2021strongly}\\
	\hline
	\end{tabular}
	\end{center}
 \vspace{-0.4cm}
\end{table}

The main contributions of this paper are summarized as follows.
\begin{itemize}
    \item For the mean estimation problem, we use a two-stage approach for $d=1$. With higher dimensionality, for $\ell_\infty$ support, our method divides users into groups, and the strategy of such grouping is tailored to the privacy level $\epsilon$. We then use Kashin's representation to obtain a tight result for $\ell_2$ support. 

    \item We apply the mean estimation to the stochastic optimization problem and derive a rate of $\tilde{O}(d/(nm(\epsilon^2\wedge\epsilon))$, matches the item-level bound in \cite{duchi2013local} under the same total sample size.

    \item For nonparametric classification and regression, we divide the support into grids and apply the Hadamard transform, which is shown to be optimal under user-level LDP. 
\end{itemize}
In general, the results show that the user-level LDP requirement is similar or sometimes even weaker than the item-level one, which is crucially different from the central model. 

\section{Related Work}

\textbf{Item-level DP.} We start with mean estimation, which is a basic but important statistical task since it serves as building blocks of stochastic optimization and deep learning \cite{abadi2016deep}, which requires estimating the mean of gradients. \cite{asi2020instance,bun2019average,huang2021instance,liu2021robust,hopkins2022efficient,vargaftik2021drive} studied mean estimation under central DP. For the local model, \cite{duchi2018minimax} introduces an order optimal mean estimation method, which is then improved in \cite{li2023robustness,feldman2021lossless}. Moreover, \cite{chen2020breaking} achieved optimal communication cost. \cite{bhowmick2018protection} proposed PrivUnit, which is then shown to be optimal in constants \cite{asi2022optimal}. \cite{asi2024fast} proposed ProjUnit, which reduces the communication complexity of PrivUnit. Mean estimation can be used in other problems. For example, in stochastic optimization, various methods have been proposed under central DP requirements \cite{chaudhuri2011differentially,bassily2014private,bassily2019private,feldman2020private,asi2021private,kamath2022improved}. Under local DP, \cite{duchi2013local} proposed a stochastic gradient method, which calculates the noisy gradient from each sample and then update the model. For nonparametric statistics, \cite{duchi2013local,duchi2018minimax} shows that the nonparametric density estimation under LDP has a convergence rate of $O((n\epsilon^2)^{-\beta/(d+\beta)})$ for small $\epsilon$, which is inevitably slower than the non-private rate $O(n^{-\beta/(d+2\beta)})$ \cite{tsybakov2009introduction}. \cite{berrett2019classification} and \cite{berrett2021strongly} extend the analysis to nonparametric classification and regression problems, respectively.

\textbf{User-level DP.} Under central model, \cite{geyer2017differentially} proposes a simple clipping method. \cite{levy2021learning} designs a two-stage approach for one-dimensional mean estimation, and then extends to higher dimension using the Hadamard transform. This method is then used in stochastic optimization problems \cite{bassily2023user,liu2024user}. \cite{zhao2024huberdp} designs a multi-dimensional Huber loss minimization approach \cite{zhao2024huber} for mean estimation under user-level DP that copes better with imbalanced users and heavy-tailed distributions. Additionally, some works are focusing on black-box conversion from item-level DP to user-level, such as \cite{ghazi2021user,bun2023stability,ghazi2023user}. Under the local model, \cite{acharya2023discrete} studies the discrete distribution estimation problem. 

To the best of our knowledge, our work is the first attempt to analyze learning with user-level local DP in general. We cover a wide range of statistical problems. Unlike the central DP, under the local model, the user-level privacy requirement is not necessarily stronger than the item-level one. Moreover, in user-level central DP, a single algorithm structure is enough. However, under the local model, to make the mean squared error optimal for all privacy levels, privacy mechanisms need to be adaptive to $\epsilon$. Therefore, the algorithm design and theoretical analysis are crucially different from the central user-level DP.

\section{Preliminaries}\label{sec:pre}


Suppose there are $n$ users, and each user has $m$ identical and independently distributed (i.i.d) samples, denoted as $\mathbf{X}_{ij}\in \mathcal{X}$, $i=1,\ldots, n$, $j=1,\ldots, m$. Let $\mathbf{X}_i=\{\mathbf{X}_{i1}, \ldots, \mathbf{X}_{im} \}$ be the set of all samples stored in user $i$. Due to privacy concerns, users are unwilling to upload $\mathbf{X}_i$ directly. Instead, there is a privacy mechanism that transforms $\mathbf{X}_1,\ldots, \mathbf{X}_n$ into $n$ random variables $\mathbf{Z}_1,\ldots, \mathbf{Z}_n\in \mathcal{Z}$ with $\mathbf{Z}_i=M_i(\mathbf{X}_i, \mathbf{Z}_1,\ldots, \mathbf{Z}_{i-1})$, in which $M_i:\mathcal{X}\times \mathcal{Z}^{i-1}\rightarrow\mathcal{Z}$ is a function with random output. The user-level LDP is defined as follows.
\begin{defi}\label{def:dp}
	Given a privacy parameter $\epsilon \geq 0$, the privacy mechanism $M_i$ is user-level $\epsilon$-LDP if for all $i$, all values of $\mathbf{z}_1,\ldots, \mathbf{z}_{i-1}$, all $\mathbf{x}, \mathbf{x}'\in \mathcal{X}^m$ and all $S\subseteq \mathcal{Z}$,
	\begin{eqnarray}
		\text{P}(\mathbf{Z}_i\in S|\mathbf{X}_i=\mathbf{x}, \mathbf{Z}_{1:i-1}=\mathbf{z}_{1:i-1})\leq e^\epsilon
		\text{P}(\mathbf{Z}_i\in S|\mathbf{X}_i=\mathbf{x}', \mathbf{Z}_{1:i-1}=\mathbf{z}_{1:i-1}),
		\label{eq:def}
	\end{eqnarray}
	in which $\mathbf{Z}_i=M_i(\mathbf{X}_i, \mathbf{Z}_1,\ldots, \mathbf{Z}_{i-1})$, $\mathbf{Z}_{1:i-1}=(\mathbf{Z}_1,\ldots, \mathbf{Z}_{i-1})$, $\mathbf{z}_{1:i-1}=(\mathbf{z}_1,\ldots, \mathbf{z}_{i-1})$.	
\end{defi}
Definition \ref{def:dp} requires that the distributions of $\mathbf{Z}_i$ should not change much even if the whole local dataset $\mathbf{X}_i=\{\mathbf{X}_{i1}, \ldots, \mathbf{X}_{im} \}$ is altered. From \eqref{eq:def}, even if the adversary can observe $\mathbf{Z}_i$, it can not infer the value of $\mathbf{X}_i$ exactly. Smaller $\epsilon$ indicates stronger privacy protection since it is harder to distinguish $\mathbf{X}_i$. The difference between item-level and user-level LDP is illustrated in Figure \ref{fig:compare}. In the item-level case, each sample is transformed into a privatized one, while in the user-level case, all samples of a user are combined to generate a privatized sample. For both item-level and user-level cases, at the final step, all privatized samples are aggregated to generate the output. One natural question is how the difficulty of achieving user-level LDP compares with the item-level counterparts. Regarding this question, we have the following statements.
\begin{wrapfigure}{r}{0.45\textwidth}
\vspace{-0.7cm}
	\begin{center}
		\includegraphics[width=1.0\linewidth]{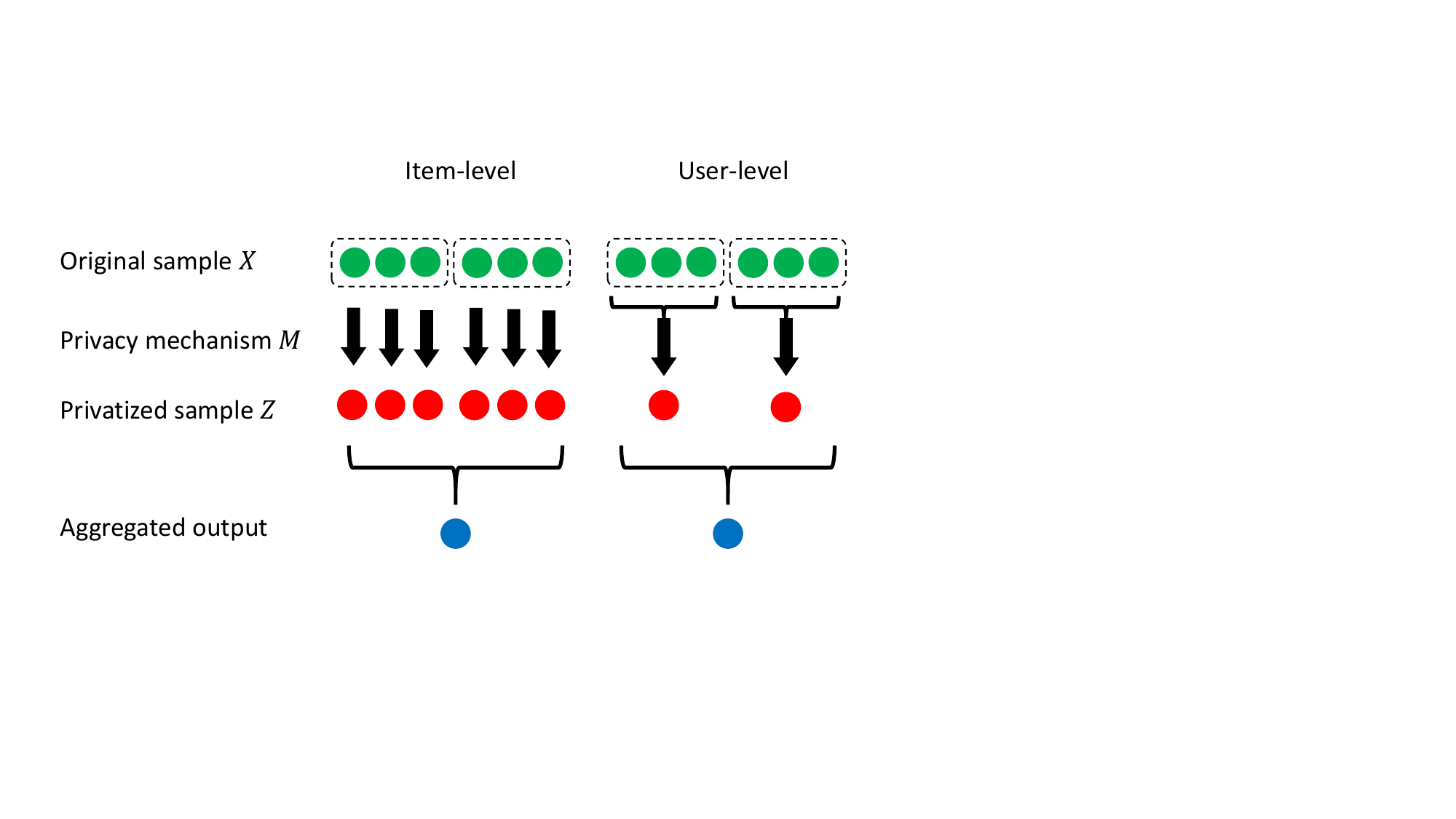}
  \vspace{-0.7cm}
		\caption{\small Comparison of item-level versus user-level LDP. Dashed rectangles represent users.}\label{fig:compare}
	\end{center}
 \vspace{-12mm}
\end{wrapfigure}
\begin{prop}\label{prop:facts}
	Based on Definition \ref{def:dp}, for any statistical problems, there are two basic facts:
	
	(1) If item-level $(\epsilon/m)$-LDP can be achieved with $nm$ samples, then user-level $\epsilon$-LDP can be achieved using $n$ users with $m$ samples per user;
	
	(2) If item-level $\epsilon$-LDP can be achieved with $n$ samples, then user-level $\epsilon$-LDP can be achieved using $n$ users with $m$ samples per user. 
\end{prop}

In the above statements, (1) holds due to the group privacy property. For (2), if a task can be solved using $n$ samples under item-level $\epsilon$-LDP, then just randomly picking a sample from each user satisfies user-level $\epsilon$-LDP. These results also suggest two baseline methods that transform item-level methods to user-level. However, these simple conversions are far from optimal. For the first one, $(\epsilon/m)$-LDP is too strong. For the second one, many samples are wasted. 

One may wonder if user-level LDP is a stronger requirement than the item-level one. In other words, if item-level $\epsilon$-LDP can be achieved with $nm$ samples, then can we achieve user-level $\epsilon$-LDP using $n$ users with $m$ samples per user? Under the central model, the answer is affirmative: user-level DP is stronger because the definition of user-level $\epsilon$-DP ensures item-level $\epsilon$-DP \cite{levy2021learning}. Nevertheless, under the local model, things become more complex. On the one hand, user-level LDP imposes stronger privacy requirements, since the distribution of the output variables can only change to a limited extent even when the local dataset is replaced as a whole. On the other hand, user-level LDP enables the encoder to obtain complete information of local datasets, thus the difficulty is somewhat reduced in this aspect. From Table \ref{tab:comp}, for many problems, with the same total sample sizes, user-level and item-level LDP yield nearly the same error bounds. If the distribution has tails, then the user-level LDP is even easier to achieve, which is perhaps surprising.



Before discussing each task in detail, we clarify some notations that will be used in subsequent sections.  Denote $a\wedge b=\min(a,b)$, $a\vee b=\max(a,b)$, and $a\lesssim b$ if there exists a constant $C$ that may depend on the constants made in problem assumptions, such that $a\leq Cb$. Conversely, $a\gtrsim b$ means $a\geq Cb$. $a\sim b$ means that $a\lesssim b$ and $a\gtrsim b$ both hold. 

\section{Mean Estimation}\label{sec:mean1}
For one-dimensional problem, we introduce a two-stage method. Despite that similar idea has also been used in central user-level DP \cite{levy2021learning}, details and theoretical analysis are different. We then extend the analysis to high-dimensional problems. To achieve optimal convergence rate for all privacy levels, our strategies are designed separately for each $\epsilon$.

\subsection{One Dimensional Case}
We start with the case such that the distribution has bounded support $\mathcal{X}=[-D, D]$ for some $D$, and introduce a two-stage method. The first stage uses half of the users to identify an interval $[L, R]$, which is much smaller than $[-D, D]$ but contains $\mu:=\mathbb{E}[X]$ with high probability. The purpose of this stage is to significantly reduce the strength of Laplacian noise needed to protect privacy, and thus reduce the negative effect on the estimation accuracy caused by privacy mechanisms. At the second stage, the algorithm then truncates the values into $[L, R]$, and adds a Laplacian noise to ensure $\epsilon$-LDP at user-level. Finally, $\mu$ can be estimated with a simple average over the other half of users. The details are provided in Algorithm \ref{alg:mean1d}.	

\begin{wrapfigure}{r}{0.58\textwidth}
\vspace{-5mm}
\begin{minipage}{0.58\textwidth}
\small
\begin{algorithm}[H]
	\caption{\small MeanEst1d: One dimensional mean estimation under user-level $\epsilon$-LDP}\label{alg:mean1d}	
	\textbf{Input:} Dataset containing $n$ users with $m$ samples per user, i.e. $X_{ij}$, $i=1,\ldots, n$, $j=1,\ldots, m$\\
	\textbf{Output:} Estimated mean $\hat{\mu}$\\
	\textbf{Parameter:} $h$, $\Delta$, $D$, $\epsilon$	
	\begin{algorithmic}[1]
		\STATE Calculate $Y_i=(1/m)\sum_{j=1}^mX_{ij}$ for $i=1,\ldots, n/2$;\label{step:gety}\\
		\STATE Divide $[-D,D]$ into $B$ bins of length $h$;\\
		\STATE $Z_{ik}=\mathbf{1}(Y_i\in B_k)+W_{ik}$ for $i=1,\ldots, n/2$, $k=1,\ldots,B$, in which $W_{ik}\sim \Lap(2/\epsilon)$;\\
		\STATE Calculate $s_k=\sum_{i=1}^{n/2}Z_{ik}$ for $k=1,\ldots,B$;\\
		\STATE Let $\hat{k}^*=\underset{k}{\arg\max}s_k$;\\
		$L=-D+(\hat{k}^*-2)h$;\\
		$R=-D+(\hat{k}^*+1)h$;\label{step:getR}\\
		\STATE $Z_i=(Y_i\vee (L-\Delta))\wedge (R+\Delta)+W_i$ for $i=n/2+1,\ldots, n$, in which $W_i\sim \Lap((3h+2\Delta)/\epsilon)$;\label{step:getZ}\\
		\STATE Calculate $\hat{\mu} = (2/n)\sum_{i=n/2+1}^{n} Z_i$;\\
		\STATE \textbf{Return} $\hat{\mu}$\label{step:return}
	\end{algorithmic}
\end{algorithm}
\end{minipage}
\vspace{-8mm}
\end{wrapfigure}

The privacy guarantee and the estimation error of Algorithm \ref{alg:mean1d} are both analyzed in Theorem \ref{thm:1dmean}. In Algorithm \ref{alg:mean1d}, $\Lap(\lambda)$ means Laplacian distribution with parameter $\lambda$, whose probability density function (pdf) is $f(u)=e^{-|u|/\lambda}/(2\lambda)$.
\begin{thm}\label{thm:1dmean}
	Algorithm \ref{alg:mean1d} is user-level $\epsilon$-LDP. If $n(\epsilon^2\wedge 1)\geq c_1\ln m$ for a constant $c_1$, then with $h=4D/\sqrt{m}$ and $\Delta=D\sqrt{\ln n/m}$, the mean squared error of one dimensional mean estimation under user-level $\epsilon$-LDP satisfies
	\begin{eqnarray}
		\mathbb{E}[(\hat{\mu}-\mu)^2]\lesssim \frac{D^2}{nm}\left(1+\frac{\ln n}{\epsilon^2}\right).
		\label{eq:1dmean}
	\end{eqnarray}
\end{thm}
The proof of Theorem \ref{thm:1dmean} is shown in Appendix \ref{sec:1d}. To begin with, in Appendix \ref{sec:stage1}, we show that $[L,R]$ contains $\mu$ with high probability. To begin with, $\hat{k}^*\in \{k^*-1,k^*,k^*+1\}$ holds with high probability, in which $k^*$ is the index of the bin containing $\mu$, i.e. $\mu\in B_{k^*}$. Let $L$ be the left bound of the $(\hat{k}^*-1)$-th bin, and $R$ be the right bound of the $(\hat{k}^*+1)$-th bin, then with high probability, $\mu\in [L,R]$. We then bound the bias and variance separately. As shown in Proposition \ref{prop:facts}, there are two baseline methods to achieve user-level LDP from item-level LDP. The first one is to achieve item-level $(\epsilon/m)$-LDP for all samples. This yields a bound $O(D^2m/(n\epsilon^2)+D^2/(nm))$. The second one is to achieve item-level $\epsilon$-LDP for $n$ samples randomly selected from $n$ users, which also only yields $O(D^2/(n(\epsilon^2\wedge 1)))$, significantly worse than the right hand side of \eqref{eq:1dmean}.

In Theorem \ref{thm:1dmean}, the requirement $n(\epsilon^2\wedge 1)\geq c_1\ln m$ is necessary since if $n$ is fixed, then the mean squared error will never converge to zero with increasing $m$. From an information-theoretic perspective, a fixed number of privatized variables can only transmit limited information \cite{cuff2016differential,wang2016relation}. Therefore, it is necessary to let $n$ grow with $m$, which is also discussed in \cite{levy2021learning} for user-level central DP. Theorem \ref{thm:1dmmx} shows the information-theoretic minimax lower bound. 
\begin{thm}\label{thm:1dmmx}
	Denote $\mathcal{P}_\mathcal{X}$ as the set of all distributions supported on $\mathcal{X}=[-D,D]$, $\mathcal{M}_\epsilon$ as all mechanisms satisfying $\epsilon$-LDP, then
	\begin{eqnarray}
		\underset{\hat{\mu}}{\inf} \underset{M\in \mathcal{M}_\epsilon}{\inf} \underset{p\in \mathcal{P}_\mathcal{X}}{\sup}\mathbb{E}\left[(\hat{\mu}-\mu)^2\right]\gtrsim \frac{D^2}{nm(\epsilon^2\wedge 1)}.
		\label{eq:1dmmx}
	\end{eqnarray}
 Moreover, with fixed $n$, the mean squared error will not converge to zero as $m$ increases. To be more precise, 
    $\mathbb{E}[(\hat{\mu}-\mu)^2]\geq (1/4) D^2 e^{-n\epsilon(e^\epsilon-1)}.$
\end{thm}

Comparison of \eqref{eq:1dmean} and \eqref{eq:1dmmx} show that the upper and lower bounds match up to a logarithm factor, thus the two-stage method is nearly minimax optimal. Finally, we extend the method to the case with unbounded support. In this case, we replace step 1 in Algorithm \ref{alg:mean1d} with $Y_i=-D\vee \left(\bar{X}_i\wedge D\right)$, in which $\bar{X}_i=(1/m)\sum_{j=1}^m X_{ij}$ is the $i$-th user-wise mean. Such clipping operation controls the sensitivity. Other steps are the same as Algorithm \ref{alg:mean1d}. The convergence rate is shown in Theorem \ref{thm:1dmeanunb}.
\begin{thm}\label{thm:1dmeanunb}
Assume that $\mathbb{E}[|X|^p]\leq M_p<\infty$ for some finite constant $M_p$, with $p\geq 2$. If $n\epsilon^2\geq c_1\ln m$, then with Algorithm \ref{alg:mean1d}, except that step 1 is replaced by $Y_i=-D\vee \left(\bar{X}_i\wedge D\right)$, the mean squared error of $\hat{\mu}$ can be bounded by
 \begin{eqnarray}
	\mathbb{E}[(\hat{\mu}-\mu)^2]\lesssim M_p^{2/p}\left[\frac{\ln m}{mn\epsilon^2}\vee (m^2n\epsilon^2)^{-\left(1-\frac{1}{p}\right)}+\frac{1}{mn}\right].\label{eq:1dmeanunb}	
\end{eqnarray}
\end{thm}
The selection of $D$ and the proof of Theorem \ref{thm:1dmeanunb} are shown in Appendix \ref{sec:1dmeanunb}. Here we provide an intuitive understanding of the phase transition in \eqref{eq:1dmeanunb}. As long as $p\geq 2$, from central limit theorem, with large $m$, similar to the case with bounded support, $Y_i$ is nearly normally distributed, and the tail is like a Gaussian distribution. Therefore, the convergence rate of the mean squared error is still $O(\ln m/(mn\epsilon^2))$, the same as the case with bounded support. However, if $m$ is small, the Gaussian approximation no longer holds. In this case, the tail of the distribution of $Y_i$ is polynomial. As a result, there is a phase transition in \eqref{eq:1dmeanunb}. Mean estimation for heavy-tailed distributions is an example that user-level LDP is easier to achieve than the item-level one. With $nm$ samples, mean squared error under item-level $\epsilon$-LDP is $O((nm\epsilon^2)^{1-1/p})$ \cite{duchi2018minimax}, significantly worse than \eqref{eq:1dmeanunb}.

\subsection{Multi-dimensional Case}\label{sec:meanhigh}

This section discusses the mean estimation problem with $d\geq 1$. Depending on the shape of the support set, the problem can be crucially different. Here we discuss two cases, i.e. $\ell_2$ support $\mathcal{X}_2=\{\mathbf{u}|\norm{\mathbf{u}}_2\leq D \}$, and $\ell_\infty$ support $\mathcal{X}_\infty=\{\mathbf{u}|\norm{\mathbf{u}}_\infty \leq D \}$. For small $\epsilon$, the mean squared error under item-level $\epsilon$-LDP is $O(d/(n(\epsilon^2\wedge \epsilon)))$  for $\ell_2$ support, and $O(d^2/(n(\epsilon^2\wedge \epsilon)))$ for $\ell_\infty$ support \cite{duchi2018minimax,asi2022optimal,feldman2021lossless,asi2024fast}. Similar to the one-dimensional case, direct transformation to user-level according to Proposition \ref{prop:facts} yields a suboptimal bound.

\textbf{$\ell_\infty$ Support.} To begin with, we focus on this relatively simpler case. The method depends on the value of $\epsilon$. Details are stated in Algorithm \ref{alg:meanhighd}.

\emph{1) High privacy ($\epsilon<1$)}. Users are assigned randomly into $d$ groups, and the $k$-th group is used to estimate $\mu_k$ (the $k$-th component of $\mu:=\mathbb{E}[\mathbf{X}]$) for $k=1,\ldots, K$. Since the size of each group is $n/d$, from \eqref{eq:1dmean}, we have
\begin{eqnarray}
	\mathbb{E}[(\hat{\mu}_k-\mu_k)^2]&\lesssim& \frac{D^2}{(n/d)m}\left(1+\frac{\ln (n/d)}{\epsilon^2}\right)
 \lesssim \frac{D^2d\ln n}{nm(\epsilon^2\wedge 1)}.
	\label{eq:high}
\end{eqnarray} 

\emph{2) Medium privacy ($1\leq \epsilon <d\ln n$)}. In this case, the privacy requirement is weaker than the case with $\epsilon<1$. Therefore, a group of users can be used to estimate more components, with $\epsilon$-LDP still satisfied. Without loss of generality, suppose that $\epsilon$ is an integer (otherwise one can just strengthen the requirement to $\lfloor \epsilon \rfloor$-LDP). In this case, users are randomly allocated to $\lceil d/\epsilon\rceil$ groups. Each group is used to estimate $\epsilon$ components, and each component is estimated under user-level $1$-LDP. From basic composition theorem \cite{dwork2010boosting}, estimating $\epsilon$ components of $\mu$ satisfies user-level $\epsilon$-LDP. Denote $n_0$ as the number of users in each group, then
\begin{eqnarray}
	\mathbb{E}[(\hat{\mu}_k-\mu_k)^2]\lesssim \frac{D^2}{n_0m}\left(1+\ln n_0\right)
 \sim\frac{D^2\ln(n\epsilon/d)}{(n\epsilon/d)m}
 \lesssim \frac{D^2d}{nm \epsilon}\ln n.
	\label{eq:medium}
\end{eqnarray}
In the first step, we replace $n$ and $\epsilon$ in \eqref{eq:1dmean} with $n_0$ and $1$ respectively, since now we are using a group with $n_0$ users to achieve $1$-LDP.

\begin{wrapfigure}{r}{0.5\textwidth}
\vspace{-0.8cm}
\begin{minipage}{0.5\textwidth}
\small
\begin{algorithm}[H]
	\caption{\small MeanEst: Multi-dimensional mean estimation under user-level $\epsilon$-LDP with $\ell_\infty$ support}\label{alg:meanhighd}
	\textbf{Input:} Dataset containing $n$ users with $m$ samples per user, i.e. $\mathbf{X}_{ij}$, $i=1,\ldots, n$, $j=1,\ldots, m$\\
	\textbf{Output:} Estimated mean $\hat{\mu}$\\
	\textbf{Parameter:} $h$, $\Delta$, $D$, $\epsilon$
	\begin{algorithmic}[1]
		\IF{$\epsilon<1$}
		\STATE Divide users randomly into $d$ groups $S_1,\ldots, S_d$;
		\FOR{$k=1,\ldots, d$}
		\STATE Estimate $\hat{\mu}_k$ with $S_k$ using Algorithm \ref{alg:mean1d} for $k=1,\ldots, d$ under $\epsilon$-LDP;
		\ENDFOR
		\ELSIF{$1\leq \epsilon < d\ln n$}
		\STATE Divide users into $\lceil d/\epsilon\rceil$ groups $S_1,\ldots, S_{\lceil d/\epsilon\rceil}$;
		\FOR{$k=1,\ldots, \lceil d/\epsilon\rceil $}
		\FOR{$l=(k-1)\epsilon+1,\ldots, k\epsilon\wedge d$}
		\STATE Estimate $\hat{\mu}_l$ with $S_k$ using Algorithm \ref{alg:mean1d} under $1$-LDP;
		\ENDFOR
		\ENDFOR
		\ELSE
		\FOR{$k=1,\ldots, d$}
		\STATE Estimate $\hat{\mu}_k$ with all users using Algorithm \ref{alg:mean1d} under $(\epsilon/d)$-LDP
		\ENDFOR
		\ENDIF
		\RETURN $\hat{\mu}=(\hat{\mu}_1,\ldots, \hat{\mu}_d)$
	\end{algorithmic}
\end{algorithm}
\end{minipage}
\vspace{-1.8cm}
\end{wrapfigure}

\emph{3) Low privacy ($\epsilon\geq d\ln n$)}. In this case, the privacy protection is much less important. We hope that the estimation error is as close to the non-private case as possible. Based on such intuition, we no longer divide users into groups. Instead, our method just estimates each component under user-level $(\epsilon/d)$-LDP, then the whole algorithm is $\epsilon$-LDP. In this case, the mean squared error of each component is bounded by
\begin{eqnarray}
	\mathbb{E}[(\hat{\mu}_k-\mu_k)^2]&\lesssim& \frac{D^2}{nm}\left(1+\frac{d^2\ln n}{\epsilon^2}\right)\nonumber\\
 &\lesssim &\frac{D^2}{nm}.
	\label{eq:low}
\end{eqnarray}
Note that $\mathbb{E}[\norm{\hat{\mu}-\mu}^2]\leq \sum_{k=1}^d \mathbb{E}[(\hat{\mu}_k-\mu_k)^2]$. A combination \eqref{eq:high}, \eqref{eq:medium} and \eqref{eq:low} yields the following theorem.

\begin{thm}\label{thm:infty}
	Under user-level $\epsilon$-LDP, if $n(\epsilon^2\wedge 1) \geq c_1d\ln m$, in which $c_1$ is the constant in Theorem \ref{thm:1dmean}, then the mean squared error of multi-dimensional mean estimation in $\mathcal{X}_\infty$ with Algorithm \ref{alg:meanhighd} is bounded by
	\begin{eqnarray}
		\mathbb{E}\left[\norm{\hat{\mu}-\mu}_2^2\right]\lesssim \frac{D^2 d}{nm}\left(1+\frac{d\ln n}{\epsilon^2\wedge \epsilon}\right).
		\label{eq:infty}
	\end{eqnarray}
\end{thm}

We would like to remark that under central DP, the loss caused by privacy mechanisms and the non-private loss are two separate terms, and we only need to select the aggregator to minimize the latter one, which does not depend on $\epsilon$. However, under the local model, privatization takes place before aggregation. Depending on $\epsilon$, the optimal randomization can be crucially different. Therefore, it is necessary to discuss each $\epsilon$ separately. In Theorem \ref{thm:infty}, we give a complete picture of the estimation error caused by different privacy levels. In particular, with $\epsilon\rightarrow\infty$, \eqref{eq:infty} converges to $D^2d/(nm)$, which is just the non-private rate. 

\textbf{$\ell_2$ Support.} Consider that $\ell_2$ support is smaller than the $\ell_\infty$ support, we expect that the bound of mean squared error can be improved over \eqref{eq:infty}. Directly applying Algorithm \ref{alg:meanhighd} does not make any improvement. Therefore, a more efficient approach is needed to achieve a better bound. Towards this goal, we use Kashin's representation \cite{lyubarskii2010uncertainty}, which has also been used in other problems related to stochastic estimation \cite{feldman2021statistical,chen2023privacy,asi2024private}. To begin with, we rephrase Kashin's representation as follows.
\begin{lem}\label{lem:kashin}
	(Kashin's representation, rephrased from Theorem 2.2 in \cite{lyubarskii2010uncertainty}) There exists a matrix $\mathbf{U}\in \mathbb{R}^{2d\times d}$ and a constant $K$, such that $\mathbf{U}^T \mathbf{U}=\mathbf{I}_d$, in which $I_d$ is the $d\times d$ identity matrix, and for all $\mathbf{x}$ with $\norm{\mathbf{x}}_2\leq 1$, $\norm{\mathbf{U}\mathbf{x}}_\infty \leq K/\sqrt{d}$.
\end{lem}
Based on Lemma \ref{lem:kashin}, our method constructs matrix $\mathbf{U}=(\mathbf{u}_1,\ldots, \mathbf{u}_{2d})^T\in \mathbb{R}^{2d\times d}$. Then we can transform all samples. Let $\mathbf{X}_{ij}'=\mathbf{U}\mathbf{X}_{ij}$ for $i=1,\ldots, n$, $j=1,\ldots, m$. Correspondingly, denote $\theta = \mathbf{U}\mu$ as the mean vector after transformation. Then $\mu$ can be estimated by estimating $\theta$ first. Since $\mathbf{X}_{ij}\in \mathcal{X}_2$, $\norm{\mathbf{X}_{ij}}_2\leq D$ holds. According to Lemma \ref{lem:kashin}, $\norm{\mathbf{X}_{ij}'}_\infty \leq KD/\sqrt{d}$. Therefore, we have transformed the $\ell_2$ support into $\ell_\infty$ support, thus $\theta$ can be estimated using Algorithm \ref{alg:meanhighd}. The only difference is that now the supremum norm is reduced from $D$ to $KD/\sqrt{d}$. After getting $\hat{\theta}$, we then transform it back to $\ell_2$ support, i.e. $\hat{\mu}=\mathbf{U}^T\hat{\theta}$. Since $\hat{\theta}$ is user-level $\epsilon$-LDP, it is guaranteed that $\hat{\mu}$ is also user-level $\epsilon$-LDP. The following theorem bounds the mean squared error of $\hat{\mu}$.
\begin{thm}\label{thm:highd}
	Under user-level $\epsilon$-LDP, if $n(\epsilon^2\wedge 1)\geq c_1d\ln m$, then the mean squared error of multi-dimensional mean estimation in $\mathcal{X}_2$ is bounded by
	\begin{eqnarray}
		\mathbb{E}\left[\norm{\hat{\mu}-\mu}_2^2\right]\lesssim \frac{D^2 }{nm}\left(1+\frac{d\ln n}{\epsilon^2\wedge \epsilon}\right).
		\label{eq:l2ub}
	\end{eqnarray}
\end{thm}

\begin{rmk}
    If the support is $\ell_1$, then we can also let $\mathbf{U}=\mathbf{H}_d/\sqrt{d}$, in which $\mathbf{H}_d$ is the $d\times d$ Hadamard matrix \cite{hedayat1978hadamard}. This can be used in the discrete distribution estimation problem. With alphabet size $A$, each sample $\mathbf{X}_{ij}$ can be viewed as a $A$ dimensional vector, such that $\mathbf{X}_{ijk}=1$ for some $k$ and $\mathbf{X}_{ijl}=0$ for $k\neq l$. Then the $\ell_2$ estimation error is bounded by $O(A\ln n/(nm(\epsilon^2\wedge \epsilon)))$, which matches \cite{acharya2023discrete} up to logarithm factor.
\end{rmk}

The corresponding minimax lower bounds are shown as follows.
\begin{thm}\label{thm:highdmmx}
	Denote $\mathcal{P}_{\mathcal{X},p}$ as the set of all distributions supported on $\mathcal{X}_p=\{\mathbf{u}|\norm{\mathbf{u}}_p\leq D \}$, $\mathcal{M}_\epsilon$ as all mechanisms satisfying user-level $\epsilon$-LDP. Then for $p\in [1,2]$, with $n$ users and $m$ samples per user,
	\begin{eqnarray}
		\underset{\hat{\mu}}{\inf}\underset{M\in \mathcal{M}_\epsilon}{\inf}\underset{p\in \mathcal{P}_{\mathcal{X}, p}}{\sup}\mathbb{E}\left[\norm{\hat{\mu}-\mu}_2^2\right]\gtrsim \frac{D^2d}{nm(\epsilon^2\wedge \epsilon)}.
		\label{eq:lplb}
	\end{eqnarray}
\end{thm}

\begin{thm}\label{thm:highdmmx2}
	Denote $\mathcal{P}_{\mathcal{X},\infty}$ as the set of all distributions supported on $\mathcal{X}_\infty$, $\mathcal{M}_\epsilon$ as all mechanisms satisfying $\epsilon$-LDP. Then with $n$ users and $m$ samples per user,
	\begin{eqnarray}
		\underset{\hat{\mu}}{\inf}\underset{M\in \mathcal{M}_\epsilon}{\inf}\underset{p\in \mathcal{P}_{\mathcal{X}, \infty}}{\sup}\mathbb{E}\left[\norm{\hat{\mu}-\mu}_2^2\right]\gtrsim \frac{D^2d^2}{nm(\epsilon^2\wedge \epsilon)}.
		\label{eq:suplb}
	\end{eqnarray}
\end{thm}
The upper bounds \eqref{eq:infty} and \eqref{eq:l2ub} match the lower bounds \eqref{eq:suplb} and \eqref{eq:lplb}. These results indicate that our methods for high dimensional mean estimation under user-level LDP are minimax optimal.

\begin{rmk}
	Now we extend the analysis to unbounded support. If $\mathbb{E}[|X_k|^p]\leq M_p$ for all $k=1,\ldots, d$, then with $n\epsilon^2\geq c_1d\ln m$ for some constant $c_1$,
 \begin{eqnarray}
	\mathbb{E}\left[\norm{\hat{\mu}-\mu}_2^2\right]\lesssim M_p^{2/p}\left[\frac{d^2\ln m}{mn(\epsilon^2\wedge \epsilon)}\vee \left(\frac{d}{m^2 n(\epsilon^2\wedge \epsilon)}\right)^{1-1/p}+\frac{d}{mn}\right].
	\label{eq:tailhighdinfty}
\end{eqnarray}
Under a stronger condition $\mathbb{E}[\norm{\mathbf{X}}_2^p]\leq M_p<\infty$, the mean squared error can be bounded by
\begin{eqnarray}
	\mathbb{E}\left[\norm{\hat{\mu}-\mu}_2^2\right]\lesssim M_p^{2/p}\left[\frac{d\ln m}{mn(\epsilon^2\wedge \epsilon)}\vee \left(\frac{d}{m^2 n(\epsilon^2\wedge \epsilon)}\right)^{1-1/p}+\frac{1}{mn}\right],
	\label{eq:tailhighd}
\end{eqnarray}
which is smaller than the rate under coordinate-wise $p$-th order bounded moment by a factor $d$. The detailed arguments can be found in Appendix \ref{sec:tailhighd}.
\end{rmk}

\section{Stochastic Optimization}\label{sec:so}
The goal is to solve the following stochastic optimization problem. Define the loss function as
$L(\theta) := \mathbb{E}[l(\mathbf{X}, \theta)]$,
in which $\mathbf{X}$ is a random variable following distribution $p$. Given $\mathbf{X}_{ij}$, $i=1,\ldots, n$, $j=1,\ldots, m$, our goal is to find the minimizer 
\begin{eqnarray}
	\theta^*=\underset{\theta \in \Theta}{\min}L(\theta).
\end{eqnarray}
The estimator is designed as follows. Users are divided randomly into $t_0$ groups. We plan to update $\theta$ in $t_0$ steps. In the $t$-th step, we use one group of users to get an estimate of $\nabla L(\theta_t) = \mathbb{E}[\nabla l(\mathbf{X}, \theta_t)]$ using Algorithm \ref{alg:meanhighd}, which includes the privacy mechanism. The result is denoted as $\mathbf{g}_t$, and the update rule of $\theta$ is
\begin{eqnarray}
	\theta_{t+1} = \theta_t-\eta \mathbf{g}_t,
\end{eqnarray} 
in which $\eta$ is the learning rate. Since Algorithm \ref{alg:meanhighd} satisfies $\epsilon$-LDP at user-level, and each user is only used once, the whole algorithm with $t_0$ steps also satisfies $\epsilon$-LDP.

These steps are summarized in Algorithm \ref{alg:so}. In step 5, the MeanEst function refers to the multi-dimensional mean estimation method shown in Algorithm \ref{alg:meanhighd}. Samples are privatized in this step. Therefore, Algorithm \ref{alg:so} satisfies user-level $\epsilon$-LDP.

\begin{wrapfigure}{r}{0.57\textwidth}
\begin{minipage}{0.57\textwidth}
\small
\begin{algorithm}[H]
	\caption{\small Stochastic optimization under user-level $\epsilon$-LDP}\label{alg:so}
	\textbf{Input:} Dataset containing $n$ users with $m$ samples per user, i.e. $\mathbf{X}_{ij}$, $i=1,\ldots, n$, $j=1,\ldots, m$\\
	\textbf{Output:} Estimated  $\hat{\theta}$
	\begin{algorithmic}[1]
		\STATE Initialize $\theta_0$;\\
		\STATE Divide users into $t_0$ groups $S_0,\ldots, S_{T-1}$;\\
		\FOR{$t=0,1,\ldots, t_0-1$}
		\STATE Calculate $\nabla l(\mathbf{X}_{ij}, \theta_t)$ for $i\in S_t$, $j=1,\ldots, m$;\\
		\STATE $\mathbf{g}_t=MeanEst(\{\nabla l(\mathbf{X}_{ij}, \theta_t)|i\in S_t, j\in [m]\})$;\label{step:est}\\
		\STATE $\theta_{t+1}=\theta_t-\eta\mathbf{g}_t$;
		\ENDFOR
		\STATE \textbf{Return} $\hat{\theta}=\theta_{t_0}$		
	\end{algorithmic}
\end{algorithm}
\end{minipage}
\vspace{-10mm}
\end{wrapfigure}
Now we provide a theoretical analysis, which is based on the following assumptions.

\begin{ass}\label{ass:so}
	(a) $l(\mathbf{X}, \theta)$ is $G$-smooth, i.e. $\nabla l(\mathbf{X}, \theta)$ is $G$-Lipschitz, in which $\nabla$ denotes the gradient with respect to $\theta$;
	
	(b) For any $\theta$, the gradient of $l$ has bounded $\ell_2$ norm with probability $1$, i.e. $\norm{\nabla l(\mathbf{X}, \theta)}_2\leq D$;
	
	(c) $L$ is $\gamma$-strong convex.
\end{ass}
The theoretical bound is shown in the following theorem.

\begin{thm}\label{thm:so}
	With $\eta\leq 1/G$, the $\ell_2$ error at $t$-th step can be bounded by
	\begin{eqnarray}
		\mathbb{E}\left[\norm{\theta_t-\theta^*}_2\right]
		\leq \left(1-\frac{1}{2}\eta\gamma\right)^t \norm{\theta_0-\theta^*}_2+\frac{2D}{\gamma}\sqrt{\frac{Ct_0}{nm}\left(1+\frac{d\ln n}{\epsilon^2\wedge \epsilon}\right)}.
		\label{eq:errt}
	\end{eqnarray}
	From \eqref{eq:errt}, there exists two constants $c_T$ and $C_T$, if $c_T\ln n\leq t_0\leq C_T\ln n$, and $n(\epsilon^2\wedge 1)\gtrsim d\ln n\ln m$, then the final estimate $\hat{\theta}=\theta_{t_0}$ satisfies
	\begin{eqnarray}
		\mathbb{E}\left[\norm{\hat{\theta}-\theta^*}_2\right]\lesssim D\sqrt{\frac{\ln n}{nm}\left(1+\frac{d\ln n}{\epsilon^2\wedge \epsilon}\right)}.
		\label{eq:sofinal}
	\end{eqnarray}
\end{thm}
The proof of Theorem \ref{thm:so} is provided in Appendix \ref{sec:sopf}. In \cite{duchi2013local}, it is shown that the bound for item-level case is $\tilde{O}(\sqrt{d/(n\epsilon^2)})$ for $\epsilon\leq 1/4$ with $n$ samples. Therefore, with the same total number of samples, our bound matches the result in \cite{duchi2013local}. 


\section{Nonparametric Classification and Regression}\label{sec:learning}
From now on, we focus on nonparametric learning problems under user-level local DP. In previous sections, the dataset contains $n$ users with $m$ samples per user, i.e. $\mathbf{X}_{ij}$, $i=1,\ldots,n$, $j=1,\ldots, m$. For nonparametric learning problems, apart from $\mathbf{X}_{ij}$, we also have the label $Y_{ij}$. Following \cite{berrett2019classification,berrett2021strongly}, which focuses on item-level classification and regression problems, suppose that $\mathbf{X}$ is supported in $[0,1]^d$, which is made for simplicity. It can be generalized to arbitrary bounded support. Denote $(\mathbf{X}, Y)$ as a test sample i.i.d to training samples, and the output of the classifier is $\hat{Y}$. 

\subsection{Classification} 
The risk is defined as
$R=\text{P}(\hat{Y}\neq Y)$.
Define
$\eta(\mathbf{x})=\mathbb{E}[Y|\mathbf{X}=\mathbf{x}]$.
Given the test sample at $\mathbf{x}$, the optimal classifier is $\hat{Y}=\sign(\eta(\mathbf{x}))$. The corresponding optimal risk, called Bayes risk, is
\begin{eqnarray}
	R^*=\text{P}(\sign(\eta(\mathbf{X}))\neq Y)=\frac{1}{2}\mathbb{E}[1-|\eta(\mathbf{X})|].
	\label{eq:rstar}
\end{eqnarray}
$\eta$ is unknown in practice. We have to learn $\eta$ from the training data. Therefore, in reality, there is inevitably a gap between the risk of a practical classifier and the Bayes risk. Such gap is called excess risk $R-R^*$. To improve the efficiency, we propose a method based on a transformation with Hadamard matrix \cite{hedayat1978hadamard}. We make some assumptions before stating our algorithm. 

\begin{ass}\label{ass:cls}
	There exists constants $C_a$, $C_b$, $f_L$, such that
	
	(a) For all $t>0$, $\text{P}(|\eta(\mathbf{X})|<t)\leq C_at^\gamma$;
	
	(b) For all $\mathbf{x}, \mathbf{x}'\in \mathcal{X}=[0,1]^d$, 
	$|\eta(\mathbf{x})-\eta(\mathbf{x}')|\leq C_b\norm{\mathbf{x}-\mathbf{x}'}_2^\beta$;

	(c) $f(\mathbf{x})\geq f_L$ for all $\mathbf{x}\in \mathcal{X}$.
\end{ass}

(a) is commonly used in many existing literatures and is typically referred to as 'Tsybakov noise condition' \cite{audibert2007fast,chaudhuri2014rates,doring2017rate}. (b) is the H{\"o}lder smoothness condition, which is commonly used in nonparametric statistics \cite{tsybakov2009introduction}. (c) is usually referred to as 'strong density assumption', which is also commonly made \cite{doring2017rate,gadat2016classification}. Our basic assumptions (a)-(c) are the same as \cite{berrett2019classification}, except that we are now considering user-level LDP, while \cite{berrett2019classification} is about item-level LDP.

\begin{thm}\label{thm:cls}
	Under Assumption \ref{ass:cls}, if $n(\epsilon^2\wedge 1)\geq c_2(\ln m+\ln n)$ for some constant $c_2$, then there exists a classifier (the algorithm is shown in Appendix \ref{sec:algcls}), such that
\begin{eqnarray}
	R-R^*\lesssim (mn(\epsilon^2\wedge \epsilon))^{-\frac{\beta(1+\gamma)}{2(d+\beta)}}\ln^{1+\gamma}n+\left(\frac{nm}{\ln n}\right)^{-\frac{\beta(1+\gamma)}{2\beta+d}}.
	\label{eq:clsub}
\end{eqnarray}
\end{thm}

The proof of Theorem \ref{thm:cls} is shown in Appendix \ref{sec:clsub}. With large $\epsilon$, \eqref{eq:clsub} reduces to $(mn/\ln n)^{-2\beta/(2\beta+d)}$, which matches the non-private rate up to logarithm factor \cite{tsybakov2009introduction}. The proof of Theorem \ref{thm:cls} is shown in Appendix \ref{sec:clsub}. The minimax bound is shown in the following theorem.

\begin{thm}\label{thm:clslb}
	Denote $\mathcal{P}_{cls}$ as the set of all distributions $p$ of $\mathbf{X}$ and regression function $\eta$ that satisfy Assumption \ref{ass:cls}, $\mathcal{M}_\epsilon$ as all mechanisms satisfying $\epsilon$-LDP, then for small $\epsilon$,
	\begin{eqnarray}
		\underset{\hat{Y}}{\inf}\underset{M\in \mathcal{M}_\epsilon}{\inf}\underset{(p,\eta)\in \mathcal{P}_{cls}}{\sup} (R-R^*)\gtrsim (nm\epsilon^2)^{-\frac{\beta(1+\gamma)}{2(d+\beta)}}+(mn)^{-\frac{\beta(1+\gamma)}{2\beta+d}}.
	\end{eqnarray}
\end{thm}

The proof of Theorem \ref{thm:clslb} is shown in Appendix \ref{sec:clslb}. The comparison of Theorem \ref{thm:cls} and Theorem \ref{thm:clslb} show that for small $\epsilon$, the upper bound and lower bound match up to a logarithmic factor. Moreover, recall \cite{berrett2019classification}, the minimax lower bound under item-level DP is $(N\epsilon^2)^{-\beta(1+\gamma)/(2(d+\beta))}$. If $N=nm$, this bound also matches \eqref{eq:clsub}, indicating that the user-level case is nearly as hard as the item-level one in asymptotic sense up to a logarithmic factor.

\subsection{Regression}\label{sec:reg}
For regression problem, we use the $\ell_2$ loss as the metric, i.e. $R=\mathbb{E}\left[(\hat{\eta}(\mathbf{X})-\eta(\mathbf{X}))^2\right]$. The support is divided similarly to classification. The bounds on the convergence rate of nonparametric regression and the corresponding minimax rate are shown in the following two theorems, respectively.
\begin{thm}\label{thm:reg}
	Under Assumption \ref{ass:cls}(b) and (c), and assume that the noise is bounded, such that with probability $1$, $|Y|<T$ for some $T$, if $n(\epsilon^2\wedge 1)\geq 2c_2(\ln m+\ln n)$, in which $c_2$ is the same constant in Theorem \ref{thm:cls}, 
	then there exists an algorithm (described in Appendix \ref{sec:algreg}), such that the risk of nonparametric regression is bounded by
	\begin{eqnarray}
		R\lesssim \left(\frac{mn(\epsilon^2\wedge \epsilon)}{\ln^2 n}\right)^{-\frac{\beta}{d+\beta}}+\left(\frac{mn}{\ln n}\right)^{-\frac{2\beta}{2\beta+d}}.
	\end{eqnarray}
\end{thm}
\begin{thm}\label{thm:reglb}
	Denote $\mathcal{P}_{reg}$ as the set of all distributions $p$ of $\mathbf{X}$ and regression function $\eta$ that satisfy the same assumption as Theorem \ref{thm:reg}, $\mathcal{Q}_\epsilon$ as all mechanisms satisfying $\epsilon$-LDP, then for small $\epsilon$,
	\begin{eqnarray}
		\underset{\hat{\eta}}{\inf}\underset{Q\in \mathcal{Q}_\epsilon}{\inf}\underset{(p,\eta)\in \mathcal{P}_{reg}}{\sup} R\gtrsim (nm\epsilon^2)^{-\frac{\beta}{d+\beta}}+(mn)^{-\frac{2\beta}{2\beta+d}}.
	\end{eqnarray}
\end{thm}	

The proof of Theorem \ref{thm:reg} and \ref{thm:reglb} are shown in Appendix \ref{sec:regub} and \ref{sec:reglb}, respectively. Similar to the classification, it can be found that the upper and lower bounds match up to logarithm factors.

\section{Conclusion}\label{sec:conc}

In this paper, we have conducted a theoretical study of various statistical problems under user-level local differential privacy, including mean estimation, stochastic optimization, nonparametric classification, and regression. For each problem, we have proposed algorithms and provided information-theoretic minimax lower bounds. The results show that for many statistical problems, with the same total sample sizes, the errors under user-level and item-level $\epsilon$-LDP are nearly of the same order.

\textbf{Limitations:} The limitation of this work includes the following aspects. Current mean estimation requires $n(\epsilon^2\wedge 1)\gtrsim d\ln m$, which may be relaxed in the future. Some assumptions can be further relaxed. For example, in classification and regression problems, we assume the pdf of $\mathbf{X}$ to be bounded away from zero, which may be unrealistic. It is possible to extend the work to heavy-tailed feature distributions, such as \cite{gadat2016classification,zhao2021minimax,cannings2020local,zhao2022analysis}. Finally, in federated learning applications, the gradient vectors are usually sparse. Therefore, sparse mean estimation under user-level LDP is worth further investigation.

\bibliographystyle{nips}
\bibliography{ldp}

\clearpage

\appendix

\section{One Dimensional Mean Estimation}\label{sec:1d}
\subsection{Privacy Guarantee}\label{sec:dp}
For $i=1,\ldots, n/2$, the privacy mechanism is shown in step 3 in Algorithm \ref{alg:mean1d}. Let $\mathbf{X}_i'=\{X_{i1}', \ldots, X_{im}' \}$ be the samples of a new user, $Z_{ik}'=\mathbf{1}(Y_i'\in B_k)+W_{ik}'$, in which $Y_i'=(\sum_{j=1}^m X_{ij}')/m$. The $\ell_1$ sensitivity can be bounded by $\norm{\mathbf{1}(Y_i\in B_k)-\mathbf{1}(Y_i'\in B_k)}_1\leq 2$. Therefore, it suffices to add a Laplacian noise with parameter $2/\epsilon$. For $i=n/2+1,\ldots, n$, the privacy mechanism is shown in step 6. Since $(R+\Delta)-(L-\Delta)=3h+2\Delta$, a laplacian noise with parameter $(3h+2\Delta)/\epsilon$ suffices to guarantee user-level $\epsilon$-LDP.

\subsection{Analysis of Stage I}\label{sec:stage1}

In this section, we prove Lemma \ref{lem:stage1}, which shows that the first stage of Algorithm \ref{alg:mean1d} successes with high probability. The precise statement of this Lemma is shown as follows.
\begin{lem}\label{lem:stage1}
	Let $h=4D/\sqrt{m}$, then with probability at least $1-\sqrt{m}e^{-c_0n(\epsilon^2\wedge 1)}$, $\mu\in [L, R]$, in which $c_0$ is a constant.
\end{lem}

Recall that for $i=1,\ldots, n$,
\begin{eqnarray}
	Y_i=\frac{1}{m}\sum_{j=1}^m X_{ij}.
\end{eqnarray}
Define $p_k=\text{P}(Y\in B_k)$, in which $Y$ denotes a random variable i.i.d with $Y_1,\ldots, Y_n$. Recall that $s_k=\sum_{i=1}^{n/2}Z_{ik}$. Then we show the following lemma.
\begin{lem}
	The following results holds. Firstly,
	\begin{eqnarray}
		\mathbb{E}[s_k] = \frac{1}{2} np_k.
	\end{eqnarray}
	Moreover, for all $t\leq  n/\sqrt{2}$,
	\begin{eqnarray}
		\text{P}(s_k-\mathbb{E}[s_k]>t)\leq \exp\left[-\frac{1}{2\left(\frac{1}{8}+\frac{8}{\epsilon^2}\right) n} t^2\right],
		\label{eq:larges}
	\end{eqnarray}
	and
	\begin{eqnarray}			\text{P}(s_k-\mathbb{E}[s_k]<-t)\leq \exp\left[-\frac{1}{2\left(\frac{1}{8}+\frac{8}{\epsilon^2}\right) n} t^2\right].
		\label{eq:smalls}
	\end{eqnarray}	
\end{lem}
\begin{proof}
	Note that
	\begin{eqnarray}
		\mathbb{E}[Z_{ik}] = \text{P}(Y\in B_k) = p_k,
		\label{eq:ez}
	\end{eqnarray}
	thus
	\begin{eqnarray}
		\mathbb{E}[s_k] = \frac{n}{2} p_k.
	\end{eqnarray}
	Now we prove \eqref{eq:larges} and \eqref{eq:smalls}. We first derive the sub-exponential parameters of $Z_{ik}$. Since $W_{ik}$ is Laplacian with parameter $b=2/\epsilon$, for $|\lambda|\leq 1/(\sqrt{2}b)=\epsilon/(2\sqrt{2})$,
	\begin{eqnarray}
		\mathbb{E}[e^{\lambda W_{ik}}]=\frac{1}{1-b^2\lambda^2}\leq e^{2b^2\lambda^2}=e^{\frac{8}{\epsilon^2} \lambda^2},
		\label{eq:wmgf}
	\end{eqnarray}
	in which the second step uses the inequality $1/(1-x)\leq e^{2x}$ for $x\leq 1/2$. Moreover,
	\begin{eqnarray}
		\mathbb{E}\left[e^{\lambda(\mathbf{1}(Y_i\in B_k)-p_k)}\right] = (1-p_k+p_ke^\lambda)e^{-\lambda p_k}.
		\label{eq:chernoff}
	\end{eqnarray}
	To bound the right hand side of \eqref{eq:chernoff}, define
	\begin{eqnarray}
		g(\lambda) = -\lambda p_k+\ln (1-p_k+p_ke^\lambda).
	\end{eqnarray}
	Then it can be shown that $g(0)=g'(0)=0$, and 
	\begin{eqnarray}
		g''(\lambda)=\frac{p_ke^\lambda(1-p_k)}{(1-p_k+p_ke^\lambda)^2}\leq \frac{1}{4}.
	\end{eqnarray}
	Therefore, \eqref{eq:chernoff} can be simplified to
	\begin{eqnarray}
		\mathbb{E}\left[e^{\lambda(\mathbf{1}(Y_i\in B_k)-p_k)}\right]\leq e^{\frac{1}{8}\lambda^2}.
		\label{eq:ymgf}	
	\end{eqnarray}
	
	From Algorithm \ref{alg:mean1d}, $Z_{ik}=\mathbf{1}(Y_i\in B_k)+W_{ik}$. Hence, for all $|\lambda|\leq \epsilon /(2\sqrt{2})$, from \eqref{eq:wmgf} and \eqref{eq:ymgf},
	\begin{eqnarray}
		\mathbb{E}[e^{\lambda(Z_{ik} - \mathbb{E}[Z_{ik}])}]\leq \exp\left[\left(\frac{1}{8}+\frac{8}{\epsilon^2}\right)\lambda^2\right].
	\end{eqnarray}
	Since $s_k=\sum_{i=1}^{n/2}Z_{ik}$, for all $|\lambda|\leq \epsilon/(2\sqrt{2})$,
	\begin{eqnarray}
		\mathbb{E}\left[e^{\lambda (s_k-\mathbb{E}[s_k])}\right]\leq \exp\left[\frac{1}{2}\left(\frac{1}{8}+\frac{8}{\epsilon^2}\right)n\lambda^2\right],
	\end{eqnarray}
	thus if $t\leq (\epsilon/8+8/\epsilon)n/(2\sqrt{2})$, 
	\begin{eqnarray}
		\text{P}(s_k-\mathbb{E}[s_k] > t)&\leq& \underset{|\lambda|\leq \epsilon/(2\sqrt{2})}{\inf} e^{-\lambda t} \exp\left[\frac{1}{2}\left(\frac{1}{8}+\frac{8}{\epsilon^2}\right)n\lambda^2\right]\nonumber\\
		&\leq &\exp\left[-\frac{1}{2\left(\frac{1}{8}+\frac{8}{\epsilon^2}\right) n} t^2\right].
	\end{eqnarray}
	Similar bound holds for $\text{P}(s_k-\mathbb{E}[s_k]< -t)$. Also note that $\epsilon/8 + 8/\epsilon \geq 2$. Therefore, \eqref{eq:larges} and \eqref{eq:smalls} are proved for $t\leq n/\sqrt{2}$.
\end{proof}
The next lemma bounds the values of $p_k$.
\begin{lem}\label{lem:pk}
	Denote $k^*$ as the bin index such that $\mu\in B_{k^*}$. Then 
	
	(1) There exists $k\in \{k^*-1,k^*, k^*+1 \}$, $p_k\geq 1/2-e^{-2}$;
	
	(2) For all $k\notin \{k^*-1,k^*, k^*+1 \}$, $p_k\leq 2e^{-8}$.
\end{lem}
\begin{proof}
	\textbf{Proof of (1) in Lemma \ref{lem:pk}}. By Hoeffding's inequality,
	\begin{eqnarray}
		\text{P}(|Y-\mu|>t)\leq 2e^{-\frac{1}{2D^2} mt^2},
	\end{eqnarray}
	thus
	\begin{eqnarray}
		\text{P}(|Y-\mu|\geq \frac{2D}{\sqrt{m}})\leq 2e^{-2}.
		\label{eq:hoeffding}
	\end{eqnarray}
	\eqref{eq:hoeffding} indicates that with probability at least $1-2e^{-2}$, $Y\in (\mu-2D/\sqrt{m}, \mu+2D/\sqrt{m})$.
	Recall that $h=4D/\sqrt{m}$. If $\mu\geq c_{k^*}$, then $(\mu-2D/\sqrt{m}, \mu+2D/\sqrt{m})\subset B_{k^*}\cup B_{k^*+1}$. Thus
	$p_{k^*}+p_{k^*+1} \geq 1-2e^{-2}$. If $\mu < c_{k^*}$, similarly, $p_{k^*}+p_{k^*-1} \geq 1-2e^{-2}$. Therefore, there exists a $k\in \{k^*-1,k^*, k^*+1 \}$, such that $p_k\geq 1/2-e^{-2}$.
	
	\textbf{Proof of (2) in Lemma \ref{lem:pk}}. For $|k-k^*|\geq 2$, 
	\begin{eqnarray}
		\underset{x\in B_k}{\inf} |x-\mu|\geq \underset{x\in B_k}{\inf}\underset{x'\in B_{k^*}}{\inf} |x-x'|\geq h.
	\end{eqnarray}
	Therefore
	\begin{eqnarray}
		p_k\leq \text{P}(|Y-\mu|>h)=\text{P}(|Y-\mu|\geq \frac{4D}{\sqrt{m}})\leq 2e^{-8}.
	\end{eqnarray}
\end{proof}
Based on Lemma \ref{lem:pk}, there exists $k_0\in \{k^*-1, k^*, k^*+1\}$ such that $p_{k_0}\geq 1/2-e^{-2}$. For all $k$ with $|k-k^*|\geq 2$,
\begin{eqnarray}
	\text{P}(\hat{k}^*=k)&\leq & \text{P}(s_k\geq s_{k_0})\nonumber\\
	&\leq & \text{P}(s_k\geq n(p_k+0.18)) + \text{P}(s_{k_0}\leq n(p_{k_0}-0.18))\nonumber\\
	&\leq & 2e^{-\frac{0.18^2}{2(1/8+8/\epsilon^2)}n}\nonumber\\
	&\leq &2e^{-c_0n\epsilon^2}. 
\end{eqnarray}
Therefore
\begin{eqnarray}
	\text{P}(|\hat{k}^*-k^*|\geq 2)\leq 2(B-1)e^{-c_0n\epsilon^2}\leq 2\left(\left\lceil \frac{1}{2}\sqrt{m}\right\rceil - 1\right)e^{-c_0n\epsilon^2}\leq \sqrt{m}e^{-c_0n\epsilon^2},
	\label{eq:stage1fail}
\end{eqnarray}
for some constant $c_0$. Therefore, with probability at least $1-\sqrt{m}e^{-c_0n\epsilon^2}$, $|\hat{k}^*-k^*|\leq 1$, i.e. $\mu\in [L, R]$. 

\subsection{Proof of Theorem \ref{thm:1dmean}}\label{sec:1dmean}
In this section, we bound the mean square error of our mean estimator. Stage I has been analyzed in Section \ref{sec:stage1}. Here we focus on Stage II.

\textbf{Bound of bias.} Let 
\begin{eqnarray}
	U = (Y\vee (L-\Delta))\wedge(R+\Delta).
	\label{eq:u}
\end{eqnarray}
Recall that in Algorithm \ref{alg:mean1d}, $Z_i=(Y_i\vee (L-\Delta))\wedge (R+\Delta)+W_i$ for $i=n/2+1,\ldots, n$. Conditional on the first $n/2$ steps in stage I, the following relation holds:
\begin{eqnarray}
	\mathbb{E}[\hat{\mu}|\mathbf{Z}_{1:n/2}] = \mathbb{E}[Z_i|\mathbf{Z}_{1:n/2}] = \mathbb{E}[U|\mathbf{Z}_{1:n/2}].
	\label{eq:equal}
\end{eqnarray}
To bound the bias of $\hat{\mu}$, it suffices to bound $|\mathbb{E}[U] - \mu|$. From \eqref{eq:u},
\begin{eqnarray}
	\mathbb{E}[U|\mathbf{Z}_{1:n/2}] &=& \mathbb{E}[Y\mathbf{1}(L-\Delta\leq Y\leq R+\Delta)|\mathbf{Z}_{1:n/2}] \nonumber\\
	&&+ (L-\Delta) \text{P}(Y<L-\Delta|\mathbf{Z}_{1:n/2}) + (R+\Delta) \text{P}(Y>R+\Delta|\mathbf{Z}_{1:n/2}).
	\label{eq:eu}
\end{eqnarray}
Moreover,
\begin{eqnarray}
	\mu &=& \mathbb{E}[Y]\nonumber\\
	&=& \mathbb{E}[Y\mathbf{1}(L-\Delta\leq Y\leq R+\Delta)] +\mathbb{E}[Y\mathbf{1}(Y<L-\Delta)] + \mathbb{E}[Y\mathbf{1}(Y>R+\Delta)].
	\label{eq:mu}
\end{eqnarray}
Note that
\begin{eqnarray}
	&&\mathbb{E}[Y\mathbf{1}(Y>R+\Delta)|\mathbf{Z}_{1:n/2}] \nonumber\\
	&=& \mathbb{E}[(Y-R-\Delta)\mathbf{1}(Y>R+\Delta)|\mathbf{Z}_{1:n/2}]+ (R+\Delta)\text{P}(Y>R+\Delta|\mathbf{Z}_{1:n/2})\nonumber\\
	&=&\int_0^\infty \text{P}(Y>R+\Delta+t|\mathbf{Z}_{1:n/2}) dt + (R+\Delta) \text{P}(Y>R+\Delta|\mathbf{Z}_{1:n/2}),
	\label{eq:right}
\end{eqnarray}
and similarly,
\begin{eqnarray}
	&&\mathbb{E}[Y\mathbf{1}(Y<L-\Delta)|\mathbf{Z}_{1:n/2}] \nonumber\\
	&=& -\mathbb{E}[(L-\Delta-Y)\mathbf{1}(Y<L-\Delta)|\mathbf{Z}_{1:n/2}]+(L-\Delta) \text{P}(Y<L-\Delta|\mathbf{Z}_{1:n/2})\nonumber\\
	&=& (L-\Delta) \text{P}(Y<L-\Delta|\mathbf{Z}_{1:n/2})-\int_0^\infty \text{P}(Y<L-\Delta-t|\mathbf{Z}_{1:n/2})dt.
	\label{eq:left}
\end{eqnarray}
From \eqref{eq:eu}, \eqref{eq:mu}, \eqref{eq:right} and \eqref{eq:left}, the bias of $\hat{\mu}$ can be bounded by
\begin{eqnarray}
	|\mathbb{E}[U]-\mu|=\left|\int_0^\infty \text{P}(Y>R+\Delta+t) dt -\int_0^\infty \text{P}(Y<L-\Delta-t)dt\right|.
	\label{eq:b1}
\end{eqnarray}
Denote $E_1$ as the event that stage I is successful, i.e. $\mu\in [L, R]$. Conditional on $E_1$,
\begin{eqnarray}
	\int_0^\infty \text{P}(Y>R+\Delta+t|E_1) dt
	&\leq & \int_0^\infty \text{P}(|Y-\mu|>R+\Delta-\mu+t|E_1) dt\nonumber\\	
	&\leq & \int_{\Delta}^\infty \text{P}(|Y-\mu|>t) dt\nonumber\\
	&\overset{(a)}{\leq} & 2\int_\Delta^\infty e^{-\frac{m}{2D^2} t^2} dt\nonumber\\
	&=& \frac{2D}{\sqrt{m}}\int_{\sqrt{m}\Delta/D}^\infty e^{-\frac{1}{2} u^2} du\nonumber\\
	&\overset{(b)}{\leq} & \frac{2\sqrt{2\pi}D}{\sqrt{m}}e^{-\frac{1}{2}\left(\frac{\sqrt{m}\Delta}{D}\right)^2}\nonumber\\
	&\overset{(c)}{=}&\frac{2\sqrt{2\pi}D}{\sqrt{mn}}.
	\label{eq:bright}
\end{eqnarray}
(a) uses Hoeffding's inequality. (b) uses the inequality $\int_s^\infty e^{-\frac{1}{2} u^2} du\leq \sqrt{2\pi}e^{-\frac{1}{2}s^2}$. For (c), recall that $\Delta = D\sqrt{\ln n/m}$. Similarly,
\begin{eqnarray}
	\int_0^\infty \text{P}(Y<L-\Delta-t|E_1)dt \leq \frac{2\sqrt{2\pi}D}{\sqrt{mn}}. 
	\label{eq:bleft}
\end{eqnarray}
Therefore, from \eqref{eq:equal}, \eqref{eq:b1}, \eqref{eq:bleft} and \eqref{eq:bright}, under $E_1$,
\begin{eqnarray}
	|\mathbb{E}[\hat{\mu}|\mathbf{Z}_{1:n/2}] - \mu|\leq \frac{4\sqrt{2\pi}D}{\sqrt{mn}}.
	\label{eq:bias}
\end{eqnarray}
If $E_1$ is not satisfied, then $|\hat{\mu}-\mu|\leq 2D$. Hence
\begin{eqnarray}
	|\mathbb{E}[\hat{\mu}] - \mu|=|\mathbb{E}[U]-\mu|\leq \frac{4\sqrt{2\pi}{D}}{\sqrt{mn}}+2D\text{P}(E_1^c),
\end{eqnarray}

\textbf{Bound of Variance.} Let $\Var[X]:=\sigma^2$. Since $X\in [-D,D]$, $\sigma^2\leq D^2$ holds. Therefore
\begin{eqnarray}
	\Var[Z_i]\leq \Var[Y]+\Var[W_i] = \frac{\sigma^2}{m} + 2\frac{(3h+2\Delta)^2}{\epsilon^2}.
\end{eqnarray}
Thus
\begin{eqnarray}
	\Var[\hat{\mu}]\leq \frac{\sigma^2}{mn}+\frac{2(3h+2\Delta)^2}{n\epsilon^2}.
	\label{eq:varmu}
\end{eqnarray}
Recall that $h=4D/\sqrt{m}$, $\Delta=D\sqrt{\ln n/m}$, $\text{P}(E_1^c)\leq \sqrt{m}e^{-c_0n\epsilon^2}$, the mean squared error can be bounded by
\begin{eqnarray}
	\mathbb{E}[(\hat{\mu}-\mu)^2]\lesssim \frac{D^2\ln n}{nm\epsilon^2}+\frac{D^2}{mn}.
\end{eqnarray}

\subsection{Proof of Theorem \ref{thm:1dmmx}}\label{sec:1dmmx}
Let $V$ be a random variable taking values in $\{-1,1\}$ with equal probability. Construct the distribution of $X$ as following:
\begin{eqnarray}
	\text{P}(X=D|V=v) = \frac{1+s v}{2}, \text{P}(X = -D) =\frac{1-s v}{2},
	\label{eq:p0}
\end{eqnarray}
in which $0<s\leq 1/2$. Define
\begin{eqnarray}
	\mu_+&=&\mathbb{E}[X|V=1],\\
	\mu_-&=&\mathbb{E}[X|V=-1],
\end{eqnarray}
then $\mu_+=Ds$, $\mu_-=-Ds$.

Denote
\begin{eqnarray}
	\hat{V}=\mathbf{1}(\hat{\mu}>0).
\end{eqnarray}
Then
\begin{eqnarray}
	\mathbb{E}[(\hat{\mu} - \mu)^2]\geq D^2s^2 \text{P}(\hat{V}\neq V).
 \label{eq:mselb}
\end{eqnarray}
Given $X_{ij}$, $i=1,\ldots, n$, $j=1,\ldots, m$, by a private mechanism, we observe $\mathbf{Z}_i$, $i=1,\ldots, n$. Denote $p_+$ and $p_-$ as the distribution of $\mathbf{Z}_i$ conditional on $V=1$ and $V=-1$, respectively. Correspondingly, let $p_+^n$ and $p_-^n$ be the joint distribution of $\mathbf{Z}_1,\ldots, \mathbf{Z}_n$. $p_{X+}$ and $p_{X-}$ denotes the distribution of $X_{ij}$ under $V=1$ and $V=-1$, respectively. $p_{X+}^m$ and $p_{X-}^m$ are the corresponding joint distribution of $X_{i1}, \ldots, X_{im}$, i.e. all samples of a user. Then
\begin{eqnarray}
	\text{P}(\hat{V}\neq V)&\overset{(a)}{\geq} & \frac{1}{2}\left(1-\mathbb{TV}(p_+^n,p_-^n)\right)\nonumber\\
	&\overset{(b)}{\geq} & \frac{1}{2}\left(1-\sqrt{\frac{1}{2}D_{KL}(p_+^n||p_-^n)}\right)\nonumber\\
	&\overset{(c)}{\geq} &\frac{1}{2}\left(1-\sqrt{\frac{1}{2}nD_{KL}(p_+||p_-)}\right)\nonumber\\
	&\overset{(d)}{\geq} & \frac{1}{2}\left(1-\sqrt{\frac{1}{2}n(e^\epsilon - 1)^2\mathbb{TV}^2 (p_{X+}^m,p_{X-}^m)}\right)\nonumber\\
	&\overset{(e)}{\geq} & \frac{1}{2}\left(1-\frac{1}{2}\sqrt{nm(e^\epsilon - 1)^2 D_{KL}(p_{X+}||p_{X-})}\right).
\end{eqnarray}
In (a), $\mathbb{TV}$ is the total variation distance. (b) uses Pinsker's inequality, and $D_{KL}$ denotes the Kullback-Leibler (KL) divergence. (c) uses the property of KL divergence. (d) comes from Theorem 1 in \cite{duchi2018minimax}. Finally, (e) uses Pinsker's inequality again.

From \eqref{eq:p0}, 
\begin{eqnarray}
	D(p_{X+}||p_{X-})=\frac{1+s}{2}\ln \frac{1+s}{1-s} + \frac{1-s}{2} \ln \frac{1-s}{1+s}=s\ln \frac{1+s}{1-s}\leq 3s^2,
\end{eqnarray}
in which the last step holds because $0<s<1/2$. Let $s\sim 1/\sqrt{nm\epsilon^2}$, then $\text{P}(\hat{V}\neq V)\sim 1$. Hence
\begin{eqnarray}
	\underset{\hat{\mu}}{\inf} \underset{Q\in \mathcal{Q}_\epsilon}{\inf}\underset{p\in \mathcal{P}_\mathcal{X}}{\sup}\mathbb{E}[(\hat{\mu}-\mu)^2]\gtrsim \frac{D^2}{nm\epsilon^2}.
\end{eqnarray}
Moreover, from standard minimax analysis for non-private problems, it can be easily shown that
\begin{eqnarray}
	\underset{\hat{\mu}}{\inf} \underset{Q\in \mathcal{Q}_\epsilon}{\inf}\underset{p\in \mathcal{P}_\mathcal{X}}{\sup}\mathbb{E}[(\hat{\mu}-\mu)^2]\gtrsim \frac{D^2}{nm}.
\end{eqnarray}

\textbf{Limit of using fixed number of users.} Finally, we prove the results for fixed $n$, which shows that zero error can not be reached even with $m\rightarrow\infty$. Recall that $p_+$ and $p_-$ are the distribution of $\mathbf{Z}_i$ conditional on $V=1$ and $V=-1$. $\mathbf{Z}_i$ is $\epsilon$-DP with respect to $\mathbf{X}_{i1},\ldots, \mathbf{X}_{im}$, thus $|\ln p_+(S)/p_-(S)|\leq \epsilon$ for all set $S$, and then it can be shown that $D_{KL}(p_+||p_-)\leq \epsilon(e^\epsilon-1)$ \cite{dwork2014algorithmic}. Therefore
\begin{eqnarray}
    \text{P}(\hat{V}\neq V)&\geq& \frac{1}{2}(1-\mathbb{TV}(p_+^n, p_-^n) \nonumber\\
    &\geq & \frac{1}{4} e^{-D_{KL}(p_+^n||p_-^n}\nonumber\\
    &=& \frac{1}{4} e^{-nD_{KL}(p_+||p_-)}\nonumber\\
    &\geq & \frac{1}{4}e^{-n\epsilon(e^\epsilon-1)}.
\end{eqnarray}

Let $s=1$ in \eqref{eq:mselb}, then
\begin{eqnarray}
    \mathbb{E}[(\hat{\mu}-\mu)^2]\geq \frac{1}{4} D^2 e^{-n\epsilon(e^\epsilon-1)}.
\end{eqnarray}
\subsection{Proof of Theorem \ref{thm:1dmeanunb}}\label{sec:1dmeanunb}

For unbounded support, the user-wise average values are clipped to $[-D,D]$, i.e.
\begin{eqnarray}
	Y_i=-D\vee \left(\frac{1}{m}\sum_{j=1}^m X_{ij}\wedge D\right),
\end{eqnarray}
which means to clip the average value of each user to $[-D,D]$. Now for simplicity, let $Y$ be a random variable i.i.d with $Y_i$, $i=1,\ldots, n$. Define
\begin{eqnarray}
	\mu_T :=\mathbb{E}[Y].
\end{eqnarray}
Recall that in Algorithm \ref{alg:mean1d}, $Z_i=(Y_i\vee (L-\Delta))\wedge (R+\Delta)+W_i$ and $\hat{\mu} = (2/n)\sum_{i=n/2+1}^{n} Z_i$. Thus
\begin{eqnarray}
	\mathbb{E}[\hat{\mu}|\mathbf{Z}_{1:n/2}] = \mathbb{E}[Z_i|\mathbf{Z}_{1:n/2}] = \mathbb{E}[U|\mathbf{Z}_{1:n/2}].
\end{eqnarray}
The bias of $\hat{\mu}$ can be bounded by
\begin{eqnarray}
	|\mathbb{E}[\hat{\mu}|\mathbf{Z}_{1:n/2}]-\mu|\leq |\mathbb{E}[U|\mathbf{Z}_{1:n/2}]-\mu_T|-|\mu_T-\mu|.
	\label{eq:biasdecomp}
\end{eqnarray}
Now we bound two terms in the right hand side of \eqref{eq:biasdecomp} separately.

\textbf{Bound of $|\mathbb{E}[U] - \mu_T|$.} 

Similar to \eqref{eq:b1}, following steps \eqref{eq:eu}, \eqref{eq:mu}, \eqref{eq:right} and \eqref{eq:left}, it can be shown that
\begin{eqnarray}
	|\mathbb{E}[U|\mathbf{Z}_{1:n/2}]-\mu_T|=\left|\int_0^\infty \text{P}(Y>R+\Delta+t) dt -\int_0^\infty \text{P}(Y<L-\Delta-t)dt\right|.
	\label{eq:b1t}
\end{eqnarray}
Denote $E_1$ as the event that stage I is successful, i.e. $\mu\in [L, R]$. To bound the right hand side of \eqref{eq:b1t}, we use the following Lemma.
\begin{lem}\label{lem:concentration}
	(Restated from Corollary 6 in \cite{bakhshizadeh2023sharp}) If $X_1,\ldots, X_m$ are $m$ i.i.d copies of random variable $X$ with $\mathbb{E}[|X|^p]\leq M_p<\infty$, $m\geq 2$, then for any constant $c$, there exists a constant $C$, such that for all $t\geq cM_p^{1/p}\sqrt{\ln m}$,
	\begin{eqnarray}
		\text{P}\left(\left|\frac{1}{m}\sum_{j=1}^m X_j-\mu\right|>t\sqrt{\frac{1}{m}}\right)\leq CM_pt^{-p} m^{-\left(\frac{p}{2}-1\right)}.
	\end{eqnarray}
\end{lem}
According to Lemma \ref{lem:concentration}, with
\begin{eqnarray}
	\Delta\geq cM_p^{1/p}\sqrt{\ln m/m},
	\label{eq:deltarange}
\end{eqnarray}
the following bound holds:
\begin{eqnarray}
	\int_0^\infty \text{P}(Y>R+\Delta+t|E_1)dt
	&\leq & \int_\Delta^\infty \text{P}(|Y-\mu|>t)dt\nonumber\\
	&\leq & \int_\Delta^\infty CM_pt^{-p}m^{-(p-1)} dt\nonumber\\
	&\leq & \frac{CM_p}{p-1}m^{-(p-1)} \Delta^{-(p-1)}.
\end{eqnarray}
Therefore from \eqref{eq:b1},
\begin{eqnarray}
	|\mathbb{E}[\hat{\mu}]-\mu_T|\leq \frac{2CM_p}{p-1}m^{-(p-1)}  \Delta^{-(p-1)}+2DP(E_1^c),
\end{eqnarray}
Similar to Lemma \ref{lem:stage1}, it can be shown that $\text{P}(E_1^c)$ decays exponentially to zero if $D\lesssim e^{c_2n\epsilon^2}$ for some constant $c_2$. 

\textbf{Bound of $|\mu_T-\mu|$}. Denote $\bar{X}$ as a random variable i.i.d with $(1/m)\sum_{j=1}^m X_{ij}$, and $Y$ can be viewed as $\bar{X}$ clipped by $[-D,D]$, i.e. $Y=-D\vee(\bar{X} \wedge D)$. Then
\begin{eqnarray}
	\mu=\mathbb{E}[\bar{X}\mathbf{1}(-D\leq \bar{X}\leq D)]+\mathbb{E}[\bar{X}\mathbf{1}(\bar{X}>D)]+\mathbb{E}[\bar{X}\mathbf{1}(\bar{X}<-D)],
\end{eqnarray}
\begin{eqnarray}
	\mu_T=\mathbb{E}[\bar{X}\mathbf{1}(-D\leq \bar{X}\leq D)]+D\text{P}(\bar{X}>D)-D\text{P}(\bar{X}<-D).
\end{eqnarray}
For sufficiently large $m,n$, $D>\mu/2$ holds, thus
\begin{eqnarray}
	\mathbb{E}[\bar{X}\mathbf{1}(\bar{X}>D)]-D\text{P}(\bar{X}>D)&=&\int_D^\infty \text{P}(\bar{X}>t)dt\nonumber\\
	&\leq & \int_D^\infty \text{P}(\bar{X}-\mu>\frac{t}{2})dt\nonumber\\
	&\leq &\int_D^\infty 2^pCM_pm^{-(p-1)}t^{-p} dt\nonumber\\
	&\lesssim & M_pm^{-(p-1)} D^{-(p-1)},
\end{eqnarray}
in which the third step uses Lemma \ref{lem:concentration}. Hence
\begin{eqnarray}
	|\mu_T-\mu|\lesssim M_pm^{-(p-1)}D^{-(p-1)}.
\end{eqnarray}

Hence from \eqref{eq:biasdecomp}, the bias can be bounded by
\begin{eqnarray}
	|\mathbb{E}[\hat{\mu}]-\mu|\lesssim M_pm^{-(p-1)}  \Delta^{-(p-1)}+M_pD^{-(p-1)}m^{-(p-1)}.
	\label{eq:biasmu}
\end{eqnarray}
For the variance of $\hat{\mu}$, \eqref{eq:varmu} still holds, i.e.
\begin{eqnarray}
	\Var[\hat{\mu}]\leq \frac{\sigma^2}{mn}+\frac{2(3h+2\Delta)^2}{n\epsilon^2}\lesssim \frac{M_p^{2/p}}{mn}+\frac{\Delta^2}{n\epsilon^2},
	\label{eq:var}
\end{eqnarray}
in which the variance is bounded using H{\"o}lder inequality. From \eqref{eq:biasmu} and \eqref{eq:var}, the mean squared error can be bounded by
\begin{eqnarray}
	\mathbb{E}[(\hat{\mu}-\mu_T)^2]\lesssim M_p^2m^{-2(p-1)} \Delta^{-2(p-1)}+M_p^2D^{-2(p-1)}m^{-2(p-1)}+\frac{\Delta^2}{n\epsilon^2}+\frac{M_p^{2/p}}{mn}.
	\label{eq:mse}
\end{eqnarray}
We pick $\delta$ to minimize the right hand side of \eqref{eq:mse}. Meanwhile, the restriction \eqref{eq:deltarange} also needs to be guaranteed. Therefore, let
\begin{eqnarray}
	\Delta=cM_p^{1/p}\sqrt{\frac{\ln m}{m}}\vee (M_p^2n\epsilon^2)^\frac{1}{2p}m^{-\left(1-\frac{1}{p}\right)}.
\end{eqnarray}
Then
\begin{eqnarray}
	\mathbb{E}[(\hat{\mu}-\mu)^2]\lesssim M_p^{2/p}\left[\frac{\ln m}{mn\epsilon^2}\vee (M_pm^2n\epsilon^2)^{-\left(1-\frac{1}{p}\right)}+ D^{-2(p-1)}m^{-2(p-1)}+\frac{1}{mn}\right].\label{eq:mseub}	
\end{eqnarray}
If $D\gtrsim \Delta$, then the second term in \eqref{eq:mseub} will not dominate. Now the proof of Theorem \ref{thm:1dmeanunb} is complete. Recall that $D\lesssim e^{c_2n\epsilon^2}$ is needed to ensure that stage $I$ success with high probability, the suitable range of $D$ is
\begin{eqnarray}
	\Delta\lesssim D\lesssim e^{c_2n\epsilon^2}.
\end{eqnarray}

\section{Multi-dimensional Mean Estimation}\label{sec:highd}
\subsection{Proof of Theorem \ref{thm:highd}}\label{sec:highdmean}
Transformation with Kashin's representation $\mathbf{X}'=\mathbf{U}\mathbf{X}$ converts $\ell_2$ support to $\ell_\infty$ support. The only difference is that now the supremum norm reduces from $D$ to $KD/\sqrt{d}$. Hence, from Theorem \ref{thm:infty},
\begin{eqnarray}
	\mathbb{E}\left[\norm{\hat{\theta}-\theta}_2^2\right]\lesssim \frac{D^2 }{nm}\left(1+\frac{d\ln n}{\epsilon^2\wedge \epsilon}\right).
	\label{eq:thetamse}
\end{eqnarray}
Recall that the final estimator is $\hat{\mu}=\mathbf{U}^T\hat{\theta}$. Moreover, by Lemma \ref{lem:kashin}, $\mathbf{U}^T\mathbf{U}=\mathbf{I}_d$. Define $v=\hat{\theta}-\mathbf{U}\mu$, then $\mathbf{U}^T\mathbf{v} = 0$. Therefore 
\begin{eqnarray}
	\mathbb{E}\left[\norm{\hat{\theta}-\theta}_2^2\right] &\overset{(a)}{=}& \mathbb{E}\left[\norm{\mathbf{U}\hat{\mu}+\mathbf{v}-\mathbf{U}\mu}_2^2\right]\nonumber\\
 &=&\mathbb{E}\left[\norm{\mathbf{U}(\hat{\mu}-\mu)}_2^2\right]+\mathbb{E}[\norm{\mathbf{v}}^2]+2\mathbb{E}\left[(\hat{\mu}-\mu)^T \mathbf{U}^T \mathbf{v}\right]\nonumber\\
 &=& \mathbb{E}\left[\norm{\mathbf{U}(\hat{\mu}-\mu)}_2^2\right]+\mathbb{E}[\norm{\mathbf{v}}^2]\nonumber\\
 &\geq &\mathbb{E}\left[\norm{\mathbf{U}(\hat{\mu}-\mu)}_2^2\right]\nonumber\\
	&=&\mathbb{E}\left[\mathbf{U} (\hat{\mu}-\mu)(\hat{\mu}-\mu)^T\mathbf{U}^T\right]\nonumber\\
	&=&\mathbb{E}\left[\tr((\hat{\mu}-\mu)(\hat{\mu}-\mu)^T \mathbf{U}^T\mathbf{U})\right]\nonumber\\
	&\overset{(b)}{=}& \mathbb{E}\left[\tr((\hat{\mu}-\mu)(\hat{\mu}-\mu)^T) \right]\nonumber\\
	&=& \mathbb{E}[\norm{\hat{\mu}-\mu}_2^2].
	\label{eq:thetatomu}
\end{eqnarray}
From \eqref{eq:thetamse},
\begin{eqnarray}
	\mathbb{E}[\norm{\hat{\mu}-\mu}_2^2]\lesssim  \frac{D^2}{nm}\left(1+\frac{d\ln n}{\epsilon^2\wedge \epsilon}\right),
\end{eqnarray}
in which (a) holds since $\theta=\mathbf{U}\mu$, and (b) uses Lemma \ref{lem:kashin}.
\subsection{Proof of Theorem \ref{thm:highdmmx}}\label{sec:highdmmx}
Denote $\mathcal{V}=\{-1,1\}^d$. For $\mathbf{v}\in \mathcal{V}$, let
\begin{eqnarray}
	\text{P}(\mathbf{X}=D\mathbf{e}_k) &=& \frac{1+sv_k}{2d},\label{eq:px1}\\
	\text{P}(\mathbf{X}=-D\mathbf{e}_k) &=&\frac{1-sv_k}{2d},
	\label{eq:px2}
\end{eqnarray}
for $k=1,\ldots,d$, in which $\mathbf{e}_k$ is the unit vector towards $k$-th coordinate, $0<s\leq 1/2$, and $v_k$ is the $k$-th element of $\mathbf{v}$. Denote $\mu_k=\mathbb{E}[\mathbf{X}\cdot \mathbf{e}_k]$ as the $k$-th component of $\mu$. Then
\begin{eqnarray}
	\mu_k=D\frac{1+sv_k}{2d} - D\frac{1-sv_k}{2d} = \frac{D}{d} sv_k.
\end{eqnarray}
Let $\hat{\mu}_k$ be the $k$-th component of $\hat{\mu}$, and
\begin{eqnarray}
	\hat{v}_k=\mathbf{1}(\hat{\mu}_k>0).
\end{eqnarray}
If $\hat{v}_k\neq v_k$, then $|\hat{\mu}_k - \mu_k|\geq Ds/d$. Hence
\begin{eqnarray}
	\mathbb{E}\left[\norm{\hat{\mu}-\mu}_2^2\right] = \mathbb{E}\left[\sum_{k=1}^d (\hat{\mu}_k-\mu_k)^2\right]\geq \frac{D^2}{d^2} s^2\mathbb{E}[\rho_H(\hat{\mathbf{v}}, \mathbf{v})],
	\label{eq:mse1}
\end{eqnarray}
in which
\begin{eqnarray}
	\rho_H(\hat{\mathbf{v}}, \mathbf{v})=\sum_{k=1}^d \mathbf{1}(\hat{v}_k\neq v_k)
\end{eqnarray}
is the Hamming distance. Therefore the minimax lower bound can be transformed to the following form:
\begin{eqnarray}
	\underset{\hat{\mu}}{\inf}\underset{Q\in \mathcal{Q}_\epsilon}{\inf}\underset{p\in \mathcal{P}_{\mathcal{X}, 1}}{\sup}\mathbb{E}\left[\norm{\hat{\mu}-\mu}_2^2\right]\geq \frac{D^2}{d^2}s^2 \underset{\hat{\mathbf{v}}}{\inf}\underset{Q\in \mathcal{Q}_\epsilon}{\inf}\underset{\mathbf{v}\in \mathcal{V}}{\sup}\mathbb{E}[\rho_H(\hat{\mathbf{v}}, \mathbf{v})].
\end{eqnarray}
Define
\begin{eqnarray}
	\delta = \underset{Q\in \mathcal{Q}_\epsilon}{\sup}\underset{\mathbf{v}, \mathbf{v}':\rho_H(\mathbf{v}, \mathbf{v}') = 1}{\max} D(p_{\mathbf{Z}|\mathbf{v}}||p_{\mathbf{Z}|\mathbf{v}'}),
\end{eqnarray}
in which $p_{\mathbf{Z}|\mathbf{v}}$ is the distribution of $\mathbf{Z}_i$ when $\mathbf{X}_{i1}, \ldots, \mathbf{X}_{im}$ are distributed according to \eqref{eq:px1} and \eqref{eq:px2}. By Theorem 2.12 (iv) in \cite{tsybakov2009introduction},
\begin{eqnarray}
	\underset{\hat{\mathbf{v}}}{\inf}\underset{Q\in \mathcal{Q}_\epsilon}{\inf}\underset{\mathbf{v}\in \mathcal{V}}{\sup}\mathbb{E}[\rho_H(\hat{\mathbf{v}}, \mathbf{v})]\geq \frac{d}{2}\max\left(\frac{1}{2} e^{-\delta}, 1-\sqrt{\frac{\delta}{2}}\right).
\end{eqnarray}
Now it remains to bound $\beta$. From Theorem 1 in \cite{duchi2018minimax},
\begin{eqnarray}
	D(p_{\mathbf{Z}|\mathbf{v}}||p_{\mathbf{Z}|\mathbf{v}'})\leq n(e^\epsilon - 1)^2\mathbb{TV}^2(p_{\mathbf{X}|\mathbf{v}}^m, p_{\mathbf{X}|\mathbf{v}'}^m).
\end{eqnarray}
To bound the total variation distance, we use a generalized version of Pinsker's inequality, stated in Lemma \ref{lem:pinsker}. Without loss of generality, suppose $\mathbf{v}, \mathbf{v}'$ is different at the first component. Then
\begin{eqnarray}
	\mathbb{TV}^2(p_{\mathbf{X}|\mathbf{v}}^m, p_{\mathbf{X}|\mathbf{v}'}^m)&\leq & \frac{1}{2}p_{\mathbf{X}|\mathbf{v}}(\{D\mathbf{e}_1,-D\mathbf{e}_1\}) D(p_{\mathbf{X}|\mathbf{v}}^m||p_{\mathbf{X}|\mathbf{v}'}^m)\nonumber\\
	&=&\frac{1}{2d}D(p_{\mathbf{X}|\mathbf{v}}^m||p_{\mathbf{X}|\mathbf{v}'}^m)\nonumber\\
	&=&\frac{m}{2d}D(p_{\mathbf{X}|\mathbf{v}}||p_{\mathbf{X}|\mathbf{v}'})\nonumber\\
	&=&\frac{m}{2d}\left(\frac{1+s}{2d} \ln \frac{1+s}{1-s}+\frac{1-s}{2d}\ln \frac{1-s}{1+s}\right)\nonumber\\
	&=&\frac{m}{2d}\frac{s}{d}\ln \frac{1+s}{1-s}\nonumber\\
	&\leq & \frac{3ms^2}{2d^2},
\end{eqnarray}
in which the last step holds since $0<s\leq 1/2$. Therefore
\begin{eqnarray}
	\delta\leq \frac{3}{2}n(e^\epsilon - 1)^2 \frac{ms^2}{d^2}.
\end{eqnarray}
To ensure $\delta\lesssim 1$, let 
\begin{eqnarray}
	s\sim \frac{d}{\sqrt{mn\epsilon^2}} \wedge 1,
\end{eqnarray}
then
\begin{eqnarray}
	\underset{\hat{\mathbf{v}}}{\inf}\underset{Q\in \mathcal{Q}_\epsilon}{\inf}\underset{\mathbf{v}\in \mathcal{V}}{\sup}\mathbb{E}[\rho_H(\hat{\mathbf{v}}, \mathbf{v})]\gtrsim d.	
\end{eqnarray}
Hence
\begin{eqnarray}
	\underset{\hat{\mu}}{\inf}\underset{Q\in \mathcal{Q}_\epsilon}{\inf}\underset{p\in \mathcal{P}_{\mathcal{X}, 1}}{\sup}\mathbb{E}\left[\norm{\hat{\mu}-\mu}_2^2\right]\gtrsim \frac{D^2}{d} s^2 \sim \frac{D^2}{d}\left(\frac{d^2}{mn\epsilon^2}\wedge 1\right)\sim \frac{D^2d}{mn\epsilon^2}\wedge \frac{D^2}{d}.	
\end{eqnarray}
Moreover, from standard minimax analysis for non-private problems \cite{tsybakov2009introduction}, it can be shown that
\begin{eqnarray}
	\underset{\hat{\mu}}{\inf}\underset{Q\in \mathcal{Q}_\epsilon}{\inf}\underset{p\in \mathcal{P}_{\mathcal{X}, 1}}{\sup}\mathbb{E}\left[\norm{\hat{\mu}-\mu}_2^2\right]\gtrsim \frac{D^2}{mn}.
\end{eqnarray}

\subsection{Proof of Theorem \ref{thm:highdmmx2}}\label{sec:highdmmx2}
Without loss of generality, suppose $d$ is a power of $2$, which enables the construction of a Hadamard matrix $\mathbf{H}_d=(\mathbf{h}_1,\ldots, \mathbf{h}_d)$ by Sylvesters' approach \cite{yarlagadda2012hadamard}. Then $\mathbf{h}_k^T\mathbf{h}_l=0$, $\forall k\neq l$ and $h_k^T h_k=d$. Denote $\mathcal{V}=\{-1,1\}^d$. For $\mathbf{v}\in \mathcal{V}$, let
\begin{eqnarray}
	\text{P}(\mathbf{X}=D\mathbf{h}_k)&=&\frac{1+sv_k}{2d},\\
	\text{P}(\mathbf{X}=-D\mathbf{h}_k)&=&\frac{1-sv_k}{2d},
\end{eqnarray}
for $k=1,\ldots, d$, $s\in (0, 1/2]$. Then
\begin{eqnarray}
	\mathbf{h}_k^T\mu_k=\mathbb{E}[\mathbf{h}_k^T \mathbf{X}] = D\mathbf{h}_k^T\mathbf{h}_k\frac{1+sv_k}{2d}-D\mathbf{h}_k^T\mathbf{h}_k\frac{1-sv_k}{2d} = Dsv_k.
\end{eqnarray}
Let
\begin{eqnarray}
	\hat{v}_k=\mathbf{1}(\mathbf{h}_k^T\hat{\mu}_k> 0).
\end{eqnarray}
If $\hat{v}_k\neq v_k$, then $|\mathbf{h}_k^T (\hat{\mu}_k-\mu_k)|>Ds$. Hence
\begin{eqnarray}
	\mathbb{E}\left[\norm{\hat{\mu}-\mu}_2^2\right]&=&\frac{1}{d}\mathbb{E}[(\hat{\mu}-\mu)^T \mathbf{H}_d\mathbf{H}_d^T(\hat{\mu}-\mu)]\nonumber\\
	&=&\frac{1}{d}\mathbb{E}\left[\sum_{k=1}^d (\mathbf{h}_k^T(\hat{\mu}_k-\mu_k))^2\right]\nonumber\\
	&\geq & \frac{D^2}{d}{s^2}\mathbb{E}[\rho_H(\hat{\mathbf{v}}, \mathbf{v})].
\end{eqnarray}
The result is $d$ times larger than \eqref{eq:mse1}. The remaining steps just follow the case with $\ell_1$ support, i.e. Section \ref{sec:highdmmx}. The result is
\begin{eqnarray}
	\underset{\hat{\mu}}{\inf}\underset{Q\in \mathcal{Q}_\epsilon}{\inf}\underset{p\in \mathcal{P}_{\mathcal{X}, \infty}}{\sup}\mathbb{E}\left[\norm{\hat{\mu}-\mu}_2^2\right]\gtrsim \frac{D^2d^2}{mn(e^\epsilon-1)^2} + \frac{D^2}{mn}.
\end{eqnarray}

\subsection{High Dimensional Mean Estimation with Heavy Tails}\label{sec:tailhighd}
We start from the case that $\mathbb{E}[|X_k|^p]\leq M_p$ for all $k$. Then follow steps from \eqref{eq:high} to \eqref{eq:low}, using Theorem \ref{thm:1dmeanunb}, the following bounds can be obtained immmediately.

If $\epsilon<1$, then
\begin{eqnarray}
	\mathbb{E}[(\hat{\mu}_k-\mu_k)^2]\lesssim M_p^{2/p}\left[\frac{d\ln m}{mn\epsilon^2}\vee \left(\frac{m^2n\epsilon^2}{d}\right)^{1-1/p} + \frac{d}{mn}\right].
\end{eqnarray}

If $1\leq \epsilon< d\ln m$, then
\begin{eqnarray}
	\mathbb{E}[(\hat{\mu}_k-\mu_k)^2]\lesssim M_p^{2/p}\left[\frac{d\ln m}{mn\epsilon}\vee \left(\frac{m^2 n\epsilon}{d}\right)^{-(1-1/p)}+\frac{d}{mn\epsilon}\right].
\end{eqnarray}

Finally, if $\epsilon\geq d\ln m$, then
\begin{eqnarray}
	\mathbb{E}[(\hat{\mu}_k-\mu_k)^2]\lesssim M_p^{2/p}\left[\frac{d^2\ln m}{mn\epsilon^2}\vee \left(\frac{m^2n\epsilon^2}{d}\right)^{1-1/p}+\frac{1}{mn}\right].
\end{eqnarray}
Combine all these three cases, we get
\begin{eqnarray}
	\mathbb{E}[(\hat{\mu}_k-\mu_k)^2]\lesssim M_p^{2/p}\left[\frac{d\ln m}{mn(\epsilon^2\wedge \epsilon)}\vee \left(\frac{d}{m^2 n(\epsilon^2\wedge \epsilon)}\right)^{1-1/p}+\frac{1}{mn}\right].
\end{eqnarray}
Therefore
\begin{eqnarray}
	\mathbb{E}\left[\norm{\hat{\mu}-\mu}_2^2\right]\lesssim M_p^{2/p}\left[\frac{d^2\ln m}{mn(\epsilon^2\wedge \epsilon)}\vee \left(\frac{d}{m^2 n(\epsilon^2\wedge \epsilon)}\right)^{1-1/p}+\frac{d}{mn}\right].
	\label{eq:inftytail}
\end{eqnarray}
Now move on to the case with $\mathbb{E}[\norm{\mathbf{X}}_2^p]\leq M_p$. Then we still conduct transformation using Kashin's representation. By Lemma \ref{lem:kashin}, 
\begin{eqnarray}
	\norm{\mathbf{U}\mathbf{x}}_\infty \leq \frac{K}{\sqrt{d}}\norm{\mathbf{x}}_2.
\end{eqnarray}
Thus
\begin{eqnarray}
	\mathbb{E}[\norm{\mathbf{U}\mathbf{X}}_\infty^p]&\leq& \frac{K^p}{d^{p/2}}\mathbb{E}[\norm{\mathbf{X}}_2^p]\nonumber\\
	&\leq & K^pM_pd^{-p/2}.
\end{eqnarray}
Therefore, for each unit vector $\mathbf{e}_k$ for the $k$-th coordinate,
\begin{eqnarray}
	\mathbb{E}[|\mathbf{e}_k^T \mathbf{U} \mathbf{X}|^p]\leq K^p M_p d^{-p/2}.
\end{eqnarray}
Let $\theta=\mathbf{U}\mu$, and then estimate $\theta$ using $\mathbf{U}\mathbf{X}_{ij}$, $i=1,\ldots, n$, $j=1,\ldots, m$. Then we replace $M_p$ in \eqref{eq:inftytail} with $K^p M_pd^{-p/2}$. Therefore
\begin{eqnarray}
	\mathbb{E}\left[\norm{\hat{\theta}-\theta}_2^2\right]\lesssim M_p^{2/p}\left[\frac{d\ln m}{mn(\epsilon^2\wedge \epsilon)}\vee \left(\frac{d}{m^2 n(\epsilon^2\wedge \epsilon)}\right)^{1-1/p}+\frac{1}{mn}\right].
\end{eqnarray}
From \eqref{eq:thetatomu},
\begin{eqnarray}
	\mathbb{E}\left[\norm{\hat{\mu}-\mu}_2^2\right]\lesssim M_p^{2/p}\left[\frac{d\ln m}{mn(\epsilon^2\wedge \epsilon)}\vee \left(\frac{d}{m^2 n(\epsilon^2\wedge \epsilon)}\right)^{1-1/p}+\frac{1}{mn}\right].
\end{eqnarray}

\section{Stochastic Optimization}\label{sec:sopf}
This section proves Theorem \ref{thm:so}. From Theorem \ref{thm:highd}, we have
\begin{eqnarray}
	\mathbb{E}\left[\norm{g_t-\nabla L(\theta_t)}_2^2\right]\leq \frac{CD^2T}{nm}\left(1+\frac{d\ln n}{\epsilon^2\wedge \epsilon}\right)
\end{eqnarray}
for some constant $C$. Recall that the update rule is
\begin{eqnarray}
	\theta_{t+1}=\theta_t-\eta \mathbf{g}_t.
\end{eqnarray}
Then
\begin{eqnarray}
	\norm{\theta_{t+1}-\theta^*}_2 &=& \norm{\theta_t-\eta \mathbf{g}_t-\theta^*}_2\nonumber\\
	&\leq & \norm{\theta_t-\eta \nabla L(\theta_t)-\theta^*}_2+\eta \norm{\nabla L(\theta_t) - \mathbf{g}_t}_2.
\end{eqnarray}
The first term can be bounded by
\begin{eqnarray}
	&&\norm{\theta_t-\eta \nabla L(\theta_t) - \theta^*}_2^2\nonumber\\
	&=& \norm{\theta_t-\theta^*}_2^2-2\eta \langle \theta_t-\theta^*, \nabla L(\theta_t)\rangle+\eta^2 \norm{\nabla L(\theta_t)}_2^2\nonumber\\
	&\overset{(a)}{\leq} &\norm{\theta_t-\theta^*}_2^2-2\eta \left(L(\theta_t)-L(\theta^*)+\frac{\gamma}{2}\norm{\theta_t-\theta^*}_2^2\right)+\eta^2\norm{\nabla L(\theta_t)}_2^2\nonumber\\
	&\overset{(b)}{\leq} & (1-\eta\gamma)\norm{\theta_t-\theta^*}_2^2 - 2\eta (L(\theta_t)-L(\theta^*))+2\eta^2 G(L(\theta_t) - L(\theta^*))\nonumber\\
	&\overset{(c)}{\leq} & (1-\eta\gamma)\norm{\theta_t-\theta^*}_2^2,
\end{eqnarray}
in which (a) uses Assumption \ref{ass:so}(c), which requires that $L$ is $\gamma$-convex. (b) uses Assumption \ref{ass:so}(a), which requires that $\nabla L$ is $G$-Lipschitz. (c) uses the condition $\eta\leq 1/G$ stated in Theorem \ref{thm:so}. Thus
\begin{eqnarray}
	\norm{\theta_t-\eta\nabla L(\theta_t)-\theta^*}_2\leq \sqrt{1-\eta\gamma}\norm{\theta_t-\theta^*}_2\leq \left(1-\frac{1}{2}\eta\gamma\right)\norm{\theta_t-\theta^*}_2.
\end{eqnarray}
Therefore
\begin{eqnarray}
	\mathbb{E}\left[\norm{\theta_{t+1}-\theta^*}_2\right]\leq \left(1-\frac{1}{2}\eta\gamma\right) \mathbb{E}\left[\norm{\theta_t-\theta^*}_2\right]+\eta D \sqrt{\frac{CT}{nm}\left(1+\frac{d\ln n}{\epsilon^2\wedge \epsilon}\right)}.
	\label{eq:iter}
\end{eqnarray}
Repeat \eqref{eq:iter} iteratively for $t=0, \ldots, T-1$. Then
\begin{eqnarray}
	\mathbb{E}\left[\norm{\theta_T-\theta^*}_2\right]\leq \left(1-\frac{1}{2}\eta\gamma\right)^T \norm{\theta_0-\theta^*}_2+\frac{2D}{\gamma}\sqrt{\frac{CT}{nm}\left(1+\frac{d\ln n}{\epsilon^2\wedge \epsilon}\right)}.
\end{eqnarray}
With $c_T\ln n\leq T\leq C_T\ln n$ for some constant $c_T$ and $C_T$,
\begin{eqnarray}
	\mathbb{E}\left[\norm{\theta_T-\theta^*}_2\right]\lesssim D\sqrt{\frac{\ln n}{nm}\left(1+\frac{d\ln n}{\epsilon^2\wedge \epsilon}\right)}.
\end{eqnarray}

\section{Nonparametric Classification}\label{sec:clspf}
\subsection{Algorithm Description}\label{sec:algcls}
We state the algorithm for $\epsilon\leq 1$ first, and then extend to larger $\epsilon$.
\begin{eqnarray}
	K=2^{\lceil \log_2 B\rceil}
	\label{eq:K}
\end{eqnarray}
be the minimum integer that is a power of $2$ and is not smaller than $B$. Denote $\mathbf{H}_K$ as the Hadamard matrix of order $K$. Define
\begin{eqnarray}
	T_k=\underset{l\in [B]:H_{kl}=1}{\cup}B_l, k=1,\ldots, K,
	\label{eq:tk}
\end{eqnarray}
and
\begin{eqnarray}
	q_k=\left\{
	\begin{array}{ccc}
		\int_{B_k} f(\mathbf{x})\eta(\mathbf{x})d\mathbf{x} &\text{if} & k=1,\ldots, B\\
		0&\text{if} & k=B+1,\ldots, K.
	\end{array}
	\right.
	\label{eq:qk}
\end{eqnarray}
Furthermore, define
\begin{eqnarray}
	Q_k=\int_{T_k} f(\mathbf{x})\eta(\mathbf{x})d\mathbf{x}-\int_{T_k^c} f(\mathbf{x})\eta(\mathbf{x})d\mathbf{x},
	\label{eq:Qk}
\end{eqnarray}
in which $T_k^c$ is the complement of $T_k$. Then
\begin{eqnarray}
	Q_k=\sum_{l\in [B]:H_{kl}=1}q_l - \sum_{l\in [B]:H_{kl} = -1}q_l=\sum_{j=1}^K H_{kl}q_l.
\end{eqnarray}
In matrix form, we have $\mathbf{Q}=\mathbf{H}_K\mathbf{q}$, in which $\mathbf{Q}=(Q_1,\ldots, Q_K)^T$, $\mathbf{q}=(q_1,\ldots, q_K)^T$. Note that
\begin{eqnarray}
	\mathbb{E}[Y_{ij}\mathbf{1}(\mathbf{X}_{ij}\in T_k)-Y_{ij}\mathbf{1}(\mathbf{X}_{ij}\in T_k^c)]=Q_k,
\end{eqnarray}
thus we can just define
\begin{eqnarray}
	U_{ijk} = Y_{ij}\mathbf{1}(\mathbf{X}_{ij}\in T_k)-Y_{ij}\mathbf{1}(\mathbf{X}_{ij}\in T_k^c),
	\label{eq:U}
\end{eqnarray}
then we have $\mathbb{E}[U_{ijk}]=Q_k$, and $|U_{ijk}|\leq 1$. Therefore, from $U_{ijk}$, we can estimate $Q_k$ using our one dimensional mean estimation method. This approach solves the issue caused by direct extension of the algorithm in \cite{berrett2019classification}. Since the bound of $|U_{ijk}|$ does not increase with $m$, the strength of noise remains the same, thus the severe loss on the accuracy can be avoided.

Based on the discussions above, our detailed algorithm is described as following, and stated precisely in Algorithm \ref{alg:cls}. Right now, we focus on the case with $\epsilon\leq 1$.

\textbf{Training.} Firstly, we divide the users randomly into $K$ groups, such that the $k$-th group is used to estimate $Q_k$ using the one dimensional mean estimation method, i.e. Algorithm \ref{alg:mean1d}, for $k=1,\ldots, K$:
\begin{eqnarray}
	\hat{Q}_k=MeanEst1d(\{U_{ijk}|i\in S_k, j\in [m] \}).
	\label{eq:Qkest}
\end{eqnarray}

$\hat{Q}_k$ with $k=1,\ldots, K$ are grouped into a vector $\hat{\mathbf{Q}}=(\hat{Q}_1,\ldots, \hat{Q}_K)^T$. Then $q_k$ can be estimated using $\hat{\mathbf{Q}}$:
\begin{eqnarray}
	\hat{\mathbf{q}}=\mathbf{H}_K^{-1}\mathbf{Q}=\frac{1}{K}\mathbf{H}_K\hat{\mathbf{Q}},
	\label{eq:qhat}
\end{eqnarray}
in which $\hat{\mathbf{q}}=(\hat{q}_1,\ldots, \hat{q}_K)^T$ is the vector containing the estimate of $q_1,\ldots, q_K$. 

Now we comment on the privacy property of the training process. Samples are privatized in step 4, which uses Algorithm \ref{alg:mean1d}. According to Theorem \ref{thm:1dmean}, with $h=4D/\sqrt{m}$ and $\Delta=D\sqrt{\ln n/m}$, this step satisfies user-level $\epsilon$-LDP, and thus the whole training process satisfies the privacy requirement.

\textbf{Prediction.} For any test sample $\mathbf{X}$, let the output be
\begin{eqnarray}
	\hat{Y} = \sum_{k=1}^B \sign(\hat{q}_k)\mathbf{1}(\mathbf{x}\in B_k).
\end{eqnarray}

Finally, we extend the algorithm to larger $\epsilon$. The idea is similar to the multi-dimensional mean estimation shown in Section \ref{sec:highdmean}. 

\emph{Medium privacy ($1\leq \epsilon<K\ln n$).} The users are divided into $\lceil K/\epsilon\rceil$ groups (instead of $K$ groups for $\epsilon\leq 1$ case). The $k$-th group is used to estimate $\epsilon$ components $Q_{(k-1)\epsilon+1}, \ldots, Q_{k\epsilon}$, under $1$-LDP for each component.

\emph{Low privacy ($\epsilon>K\ln n$).} In this case, do not divide users into groups. Just estimate each $Q_k$ under $\epsilon/K$-LDP.

\begin{algorithm}[h!]
	\caption{Training algorithm of nonparametric classification under user-level $\epsilon$-LDP}\label{alg:cls}
	\textbf{Input:} Training dataset containing $n$ users with $m$ samples per user, i.e. $(\mathbf{X}_{ij}, Y_{ij})$, $i=1,\ldots, n$, $j=1,\ldots, m$\\
	\textbf{Output:} $\hat{\mathbf{q}}$\\
	\textbf{Parameter:} $h$, $\Delta$, $l$
	\begin{algorithmic}[1]	 
		\STATE Divide $\mathcal{X}=[0,1]^d$ into $B$ bins, such that the length of each bin is $l$;\\
		
		\STATE $K=2^{\lceil \log_2 B\rceil}$;\\
		
		\STATE Calculate $U_{ijk}$ according to \eqref{eq:U}, for $i=1,\ldots, n$, $j=1,\ldots, m$, $k=1,\ldots, K$; \\
		
		\STATE Estimate $\hat{Q}_k$ according to \eqref{eq:Qkest}, with parameters $h$ and $\Delta$, for $k=1,\ldots, K$;\label{step:Qkest}\\
		
		\STATE $\hat{\mathbf{q}}=\mathbf{H}_K\hat{\mathbf{Q}}/K$, in which $\hat{\mathbf{Q}}=(\hat{Q}_1,\ldots, \hat{Q}_K)^T$;\\
		
		\STATE \textbf{Return} $\hat{\mathbf{q}}$
		
	\end{algorithmic}
	
\end{algorithm}

\subsection{Proof of Theorem \ref{thm:cls}}\label{sec:clsub}
To begin with, we show a concentration inequality of one dimensional mean estimation.

\begin{lem}\label{lem:muerr}
	Let $E_1$ be the event that stage I is successful, i.e. $\mu\in [L,R]$. For any $t\leq \sqrt{2} (3h+2\Delta)$, in which $h=4D/\sqrt{m}$ and $\Delta=D\sqrt{\ln(Kn)/m}$, then the following bound holds:
	\begin{eqnarray}
		\text{P}(|\hat{\mu}-\mu|>t|E_1)\leq 2\exp\left[-\frac{n\left(t-4\sqrt{2\pi}\frac{D}{\sqrt{mnK}}\right)^2}{2\left(\frac{1}{4}+\frac{4}{\epsilon^2}\right)(3h+2\Delta)^2}\right].
	\end{eqnarray} 
\end{lem}

\begin{proof}
	Define $a=3h+2\Delta$ for convenience. For $i=n/2, \ldots, n$, since $W_i\sim \Lap(a/\epsilon)$,
	\begin{eqnarray}
		\mathbb{E}[e^{\lambda W_i}|E_1]\leq \exp\left[2\left(\frac{a}{\epsilon}\right)^2\lambda^2\right], \forall \lambda^2\leq \frac{\epsilon^2}{2a^2}.
	\end{eqnarray}
	Similar to \eqref{eq:ymgf}, it can be shown that
	\begin{eqnarray}
		\mathbb{E}\left[\exp\left[(Y_i\vee (L-\Delta))\wedge (L+\Delta)-\mathbb{E}[(Y_i\vee (L-\Delta))\wedge (L+\Delta)]\right]\right]\leq e^{\frac{1}{8}\lambda^2a^2}.
	\end{eqnarray}
	Note that $Z_{ik}=\mathbf{1}(Y_i\in B_k)+W_{ik}$, thus for $i=n/2, \ldots, n$,
	\begin{eqnarray}
		\mathbb{E}[e^{\lambda(Z_i-\mathbb{E}[Z_i])}|E_1]\leq \exp\left[\left(\frac{1}{8}+\frac{2}{\epsilon^2}\right) a^2\lambda^2\right], \forall \lambda^2 \leq \frac{\epsilon^2}{2a^2}.
	\end{eqnarray}
	Recall that $\hat{\mu}=(2/n)\sum_{i=n/2+1}^n Z_i$,
	\begin{eqnarray}
		\mathbb{E}\left[e^{\lambda(\hat{\mu}-\mathbb{E}[\hat{\mu}])}|E_1\right]\leq \exp\left[\left(\frac{1}{8}+\frac{2}{\epsilon^2}\right)\frac{2a^2\lambda^2}{n}\right], \forall\lambda^2\leq \frac{n^2\epsilon^2}{8a^2}.
	\end{eqnarray}
	Hence
	\begin{eqnarray}
		\text{P}(\hat{\mu}-\mathbb{E}[\hat{\mu}]>t|E_1)\leq \underset{|\lambda|\leq n\epsilon/(2\sqrt{2}a)}{\inf}\exp\left[-\lambda t+\left(\frac{1}{8}+\frac{2}{\epsilon^2}\right)\frac{2a^2\lambda^2}{n}\right].
		\label{eq:muhat}
	\end{eqnarray}
	If
	\begin{eqnarray}
		t\leq \frac{\epsilon}{\sqrt{2}}\left(\frac{1}{4}+\frac{4}{\epsilon^2}\right) a,
		\label{eq:tcond}
	\end{eqnarray}
	then the right hand side of \eqref{eq:muhat} reaches minimum at 
	\begin{eqnarray}
		\lambda^* = \frac{nt}{2\left(\frac{1}{4}+\frac{4}{\epsilon^2}\right)a^2}.
	\end{eqnarray}
	The condition \eqref{eq:tcond} can be simplified to $t\leq \sqrt{2}a$. It remains to consider the estimation bias. Following arguments similar to those used to derive \eqref{eq:bias}, with $h=4D/\sqrt{m}$ and $\Delta=D\sqrt{\ln(Kn)/m}$, the bias is bounded b$4\sqrt{2\pi}D/\sqrt{mnK}$. Therefore
	\begin{eqnarray}
		\text{P}\left(|\hat{\mu}-\mu|>t|E_1\right)\leq 2\exp\left[-\frac{n}{2\left(\frac{1}{4}+\frac{4}{\epsilon^2}\right) a^2}\left(t-\frac{4\sqrt{2\pi} D}{\sqrt{mnK}}\right)^2\right].
	\end{eqnarray}
	The proof is complete.
\end{proof}
Now we focus on the case with $\epsilon\leq 1$. Denote $E_{1k}$ as the event that the first stage is successful for estimating $\hat{q}_k$, and $E_1=\cap_k E_{1k}$. Recall \eqref{eq:Qkest} estimates $Q_k$ using Algorithm \ref{alg:mean1d}. From Lemma \ref{lem:muerr}, the following lemma can be proved easily:
\begin{lem}\label{lem:qerr}
	There exists two constants $C_1$, $C_2$, such that
	\begin{eqnarray}
		\text{P}(|\hat{q}_k-q_k|>t|E_1)\leq 2\exp\left[-C_1\frac{mn\epsilon^2}{\ln (nK)}\left(t-C_2\sqrt{\frac{1}{mn}}\right)^2\right].
		\label{eq:qerr}
	\end{eqnarray}
\end{lem}
\begin{proof}
	The size of the $k$-th group is $|S_k|=n/K$, from Lemma \ref{lem:muerr}, since $U_{ijk}$ in \eqref{eq:U} satisfies $|U_{ijk}|\leq 1$, the following bound holds:
	\begin{eqnarray}
		\text{P}(|\hat{Q}_k-Q_k|>t|E_1)\leq 2\exp\left[-\frac{n\left(t-4\sqrt{2\pi}\sqrt{\frac{1}{mn}}\right)^2}{2K\left(\frac{1}{4}+\frac{4}{\epsilon^2}\right)(3h+2\Delta)^2}\right].
	\end{eqnarray}
	From \eqref{eq:qhat},
	\begin{eqnarray}
		|\hat{q}_k-q_k|=\left|\frac{1}{K}\sum_{l=1}^K \mathbf{H}_{kl}(\hat{Q}_l-Q_l)\right|
	\end{eqnarray}
	Note that $\hat{Q}_l$ are independent for different $l$, and the values of $\mathbf{H}_{kl}$ are either $1$ or $-1$. Moreover, as discussed in Section \ref{sec:mean1}, $h\sim 1/\sqrt{m}$, $\Delta\sim \sqrt{\ln n/m}$, there exists a constant $C_1$ and $C_2$ such that \eqref{eq:qerr} holds.
\end{proof}

From \eqref{eq:stage1fail}, the failure probability of the first stage is bounded by
\begin{eqnarray}
	\text{P}(E_{1k}^c)\leq \sqrt{m} e^{-c_0n\epsilon^2}.
	\label{eq:pekc}
\end{eqnarray}
We then bound the excess risk of classification. Suppose $\mathbf{x}\in B_k$. Then given $\mathbf{x}$ and $\hat{q}_k$ obtained from training samples,
\begin{eqnarray}
	\text{P}(\hat{Y}\neq Y|\mathbf{x}, \hat{q}_k)&=&\text{P}(Y\neq \sign(\hat{q}_k))\nonumber\\
	&\leq & \mathbf{1}(\sign(\hat{q}_k)\neq \sign(\eta(\mathbf{x})))\text{P}(Y=\sign(\eta(\mathbf{x})))\nonumber\\
	&&+\mathbf{1}(\sign(\hat{q}_k)= \sign(\eta(\mathbf{x})))\text{P}(Y\neq\sign(\eta(\mathbf{x})))\nonumber\\
	&=&\mathbf{1}(\sign(\hat{q}_k)\neq \sign(\eta(\mathbf{x})))\frac{|\eta(\mathbf{x})|+1}{2}+\mathbf{1}(\sign(\hat{q}_k)= \sign(\eta(\mathbf{x})))\frac{1-|\eta(\mathbf{x})|}{2}\nonumber\\
	&=&\frac{1-|\eta(\mathbf{x})|}{2}+|\eta(\mathbf{x})|\mathbf{1}(\sign(\hat{q}_k)\neq \sign(\eta(\mathbf{x}))).
\end{eqnarray}
Therefore
\begin{eqnarray}
	R=\mathbb{E}\left[\frac{1-|\eta(\mathbf{X})|}{2}\right]+\mathbb{E}\left[\sum_{k=1}^B \int_{B_k} |\eta(\mathbf{x})|\mathbf{1}(\sign(\hat{q}_k)\neq \sign(\eta(\mathbf{x})))f(\mathbf{x})d\mathbf{x}\right].
\end{eqnarray}
Recall that the Bayes risk is
\begin{eqnarray}
	R^*=\mathbb{E}\left[\frac{1-|\eta(\mathbf{X})|}{2}\right],
\end{eqnarray}
thus the excess risk is
\begin{eqnarray}
	R-R^* =\mathbb{E}\left[\sum_{k=1}^B \int_{B_k} |\eta(\mathbf{x})|\mathbf{1}(\sign(\hat{q}_k)\neq \sign(\eta(\mathbf{x})))f(\mathbf{x})d\mathbf{x}\right].
\end{eqnarray}
Define
\begin{eqnarray}
	\eta_0=2C_bd^\frac{\beta}{2}l^\beta + \frac{2C_2}{f_Ll^d}\sqrt{\frac{1}{mn}}.
	\label{eq:eta0}
\end{eqnarray}
If $\eta(\mathbf{x})>\eta_0$, then
\begin{eqnarray}
	q_k=\int_{B_k}f(\mathbf{x})\eta(\mathbf{x})d\mathbf{x}\geq \left(\int_{B_k} f(\mathbf{x})d\mathbf{x}\right) \left(\eta(\mathbf{x})-C_bd^\frac{\beta}{2}l^\beta\right)>0.
\end{eqnarray}
Similarly, if $\eta(\mathbf{x})<-\eta_0$, $q_k<0$. Thus $\sign(\eta(\mathbf{x}))=\sign(q_k)$ if $|\eta(\mathbf{x})|>\eta_0$. Therefore, for all $\mathbf{x}$ such that $\eta(\mathbf{x})>\eta_0$,
\begin{eqnarray}
	\text{P}(\sign(\hat{q}_k)\neq \sign(\eta(\mathbf{x})))&\leq & \text{P}(\sign(\hat{q}_k)\neq \sign(q_k))\nonumber\\
	&\leq & \text{P}(E_1^c)+\text{P}(|\hat{q}_k-q_k|>|q_k||E_1)\nonumber\\
	&\overset{(a)}{\leq} & \sqrt{m}e^{-c_0n\epsilon^2}+ 2\exp\left[-C_1\frac{mn\epsilon^2}{\ln n}\left(|q_k|-\frac{C_2}{\sqrt{mn}}\right)^2\right]\nonumber\\
	&\overset{(b)}{\leq} & \sqrt{m}e^{-c_0n\epsilon^2}+2\exp\left[-\frac{1}{4}C_1\frac{mn\epsilon^2}{\ln n}|q_k|^2\right]\nonumber\\
	&\overset{(c)}{\leq}&\sqrt{m}e^{-c_0n\epsilon^2}+2\exp\left[-\frac{1}{16}C_1f_L^2\frac{mn\epsilon^2}{\ln n}\eta^2(\mathbf{x})l^{2d}\right].
	\label{eq:pqerr}
\end{eqnarray}

Now we explain (a)-(c) in \eqref{eq:pqerr}. (a) uses \eqref{eq:pekc} and Lemma \ref{lem:qerr}. For (b), note that with $\eta(\mathbf{x})>\eta_0$,
\begin{eqnarray}
	|q_k|&=&\left|\int_{B_k} \eta(\mathbf{x})f(\mathbf{x})d\mathbf{x}\right|\nonumber\\
	&\geq & |\eta(\mathbf{x})-C_b d^\frac{\beta}{2}l^\beta|\int_{B_k} f(\mathbf{x})d\mathbf{x}\nonumber\\
	&\geq &(\eta_0-C_bd^\frac{\beta}{2}l^\beta)\int_{B_k}f(\mathbf{x})d\mathbf{x}\nonumber\\
	&\geq & \frac{2C_2}{f_Ll^d}\sqrt{\frac{1}{mn}}\int_{B_k}f(\mathbf{x})d\mathbf{x}\nonumber\\
	&\geq & 2C_2\sqrt{\frac{1}{mn}}.
\end{eqnarray}
Thus
\begin{eqnarray}
	|q_k|-\frac{C_2}{\sqrt{mn}}\geq \frac{1}{2}|q_k|.
\end{eqnarray}
For (c), since $|\eta(\mathbf{x})|>\eta_0>2C_bd^\frac{\beta}{2}l^\beta$, 
\begin{eqnarray}
	\eta(\mathbf{x})-C_bd^\frac{\beta}{2}l^\beta>\frac{1}{2}\eta(\mathbf{x}).
\end{eqnarray}
Hence
\begin{eqnarray}
	|q_k|\geq \frac{1}{2}\eta(\mathbf{x})\int_{B_k} f(\mathbf{x})d\mathbf{x}\geq \frac{1}{2}\eta(\mathbf{x})f_Ll^d.
\end{eqnarray}
The proof of \eqref{eq:pqerr} (a)-(c) are complete. For $\mathbf{x}\in B_k$, denote $\hat{\eta}(\mathbf{x})=q_k$. Then based on \eqref{eq:pqerr},
\begin{eqnarray}
	R-R^*&=&\sum_{k=1}^B\int_{B_k}|\eta(\mathbf{x})|\text{P}(\sign(\hat{q}_k)\neq \sign(\eta(\mathbf{x})))f(\mathbf{x})d\mathbf{x}\nonumber\\
	&=&\int_{\eta(\mathbf{x})\leq \eta_0} \eta_0f(\mathbf{x})d\mathbf{x} + \int_{\eta(\mathbf{x})>\eta_0}|\eta(\mathbf{x})|\text{P}(\sign(\hat{\eta}(\mathbf{x}))\neq\sign(\eta(\mathbf{x})))f(\mathbf{x})d\mathbf{x}\nonumber\\
	&\leq & \eta_0\text{P}(\eta(\mathbf{X})<\eta_0)+2\mathbb{E}\left[|\eta(\mathbf{X})|\exp\left[-\frac{1}{16}C_1f_L^2 \frac{mn\epsilon^2}{\ln n}\eta^2(\mathbf{x})l^{2d}\right]\right]+\sqrt{m}e^{-c_0n\epsilon^2}.\nonumber\\
	\label{eq:riskbound}
\end{eqnarray}
For the first term in \eqref{eq:riskbound}, use Assumption \ref{ass:cls}, we have
\begin{eqnarray}
	\text{P}(\eta(\mathbf{X})<\eta_0)\lesssim \eta_0^\gamma.
\end{eqnarray}
For the second term, we can bound it with Lemma \ref{lem:tail}. The third term decays exponentially with $n$. Therefore, with $n\epsilon^2\gtrsim \ln m$, we have
\begin{eqnarray}
	R-R^*&\lesssim &\eta_0^{1+\gamma}+\left(\frac{mn\epsilon^2}{\ln n}l^{2d}\right)^{-\frac{1}{2}(1+\gamma)}\nonumber\\
	&\sim & \left(l^\beta+\frac{1}{l^d}\sqrt{\frac{1}{mn}}+\frac{\ln n}{\sqrt{mn\epsilon^2}l^d}\right)^{1+\gamma},
\end{eqnarray}
in which the second step uses \eqref{eq:eta0}. Let 
\begin{eqnarray}
	l\sim (mn\epsilon^2)^{-\frac{1}{2(d+\beta)}},
\end{eqnarray}
then
\begin{eqnarray}
	R-R^*\lesssim (mn\epsilon^2)^{-\frac{\beta(1+\gamma)}{2(d+\beta)}}\ln^{1+\gamma}n.
	\label{eq:risklow}
\end{eqnarray}
Now the proof of the bound of mean squared error for $\epsilon\leq 1$ is finished. It remains to show the case with $\epsilon>1$.

\emph{1) Medium privacy ($1\leq \epsilon<K\ln n$)}. Note that now the size of each group is $n/\lceil K/\epsilon\rceil$. Following the arguments above, it can be shown that with $l\sim (mn\epsilon^2)^{-\frac{1}{2(d+\beta)}}$,
\begin{eqnarray}
	R-R^*\lesssim (mn\epsilon)^{-\frac{\beta(1+\gamma)}{2(d+\beta)}}\ln^{1+\gamma}n.
	\label{eq:riskmed}
\end{eqnarray}

\emph{2) Low privacy ($\epsilon\geq K\ln n$)}. Now \eqref{eq:qerr} becomes
\begin{eqnarray}
	\text{P}(|\hat{q}_k-q_k|>t|E_1)\leq 2\exp\left[-C_1\frac{mnK}{\ln n}\left(t-C_2\sqrt{\frac{1}{mnK}}\right)^2\right].
\end{eqnarray}
Following previous arguments,
\begin{eqnarray}
	R-R^*\lesssim \left(l^\beta+\frac{1}{l^d}\sqrt{\frac{\ln n}{nmK}}\right)^{1+\gamma}.
\end{eqnarray}
With $l\sim (nm/\ln n)^{-1/(2\beta+d)}$,
\begin{eqnarray}
	R-R^*\lesssim \left(\frac{nm}{\ln n}\right)^{-\frac{\beta(1+\gamma)}{2\beta+d}}.
	\label{eq:riskhigh}
\end{eqnarray}

Combine \eqref{eq:risklow}, \eqref{eq:riskmed} and \eqref{eq:riskhigh}, the final bound on mean squared error is
\begin{eqnarray}
	R-R^*\lesssim (mn(\epsilon^2\wedge \epsilon))^{-\frac{\beta(1+\gamma)}{2(d+\beta)}}\ln^{1+\gamma}n+\left(\frac{nm}{\ln n}\right)^{-\frac{\beta(1+\gamma)}{2\beta+d}}.
\end{eqnarray}

\subsection{Proof of Theorem \ref{thm:clslb}}\label{sec:clslb}
Divide the whole support into $B$ bins, and the length of each bin is $l$. Then $Bl^d=1$. Let the pdf of $\mathbf{X}$ be uniform, i.e. $f(\mathbf{x}) = c$ for some constant $c$. Moreover, let $\phi(\mathbf{u})$ be some function supported at $[-1/2,1/2]^d$, such that $\phi(\mathbf{u})\geq 0$ and $\phi(\mathbf{u})l^\beta\leq 1/2$ always hold, and for any $\mathbf{x}$ and $\mathbf{x}'$,
\begin{eqnarray}
	\norm{\phi(\mathbf{u})-\phi(\mathbf{u}')}\leq C_b\norm{\mathbf{u}-\mathbf{u}'}^\beta.
\end{eqnarray}
Moreover, denote $\mathbf{c}_1,\ldots, \mathbf{c}_K$ be centers of $K$ bins, $K<B$. For $\mathbf{v}\in\mathcal{V}:= \{-1,1\}^K$, let
\begin{eqnarray}
	\eta_\mathbf{v}(\mathbf{x})=\sum_{k=1}^K v_k\phi\left(\frac{\mathbf{x}-\mathbf{c}_k}{l}\right) l^\beta.
\end{eqnarray}
For other $B-K$ bins, $\eta(\mathbf{x})=0$. It can be proved that there exists a constant $C_K$, such that if $K\leq C_K l^{\gamma \beta - d}$, then $\eta(\mathbf{x})$ satisfies Assumption \ref{ass:cls}. 

Denote 
\begin{eqnarray}
	\hat{v}_k=\underset{s\in \{-1,1\}}{\arg\max}\int_{B_k} \phi\left(\frac{\mathbf{x}-\mathbf{c}_k}{l}\right) \mathbf{1}(\sign(\hat{\eta}(\mathbf{x})) = s)f(\mathbf{x})d\mathbf{x}.
\end{eqnarray}
If $\hat{v}_k\neq v_k$, then
\begin{eqnarray}
	\int_{B_k} \phi\left(\frac{\mathbf{x}-\mathbf{c}_k}{l}\right) \mathbf{1}(\sign(\hat{\eta}(\mathbf{x}))=v_k) f(\mathbf{x})d\mathbf{x}\leq \int_{B_k} \phi\left(\frac{\mathbf{x}-\mathbf{c}_k}{l}\right) \mathbf{1}(\sign(\hat{\eta}(\mathbf{x}))=-v_k) f(\mathbf{x})d\mathbf{x}.
	\label{eq:lb1}
\end{eqnarray}
Note that
\begin{eqnarray}
	&&\int_{B_k} \phi\left(\frac{\mathbf{x}-\mathbf{c}_k}{l}\right)\left[\mathbf{1}(\sign(\hat{\eta}(\mathbf{x}))=v_k)+\mathbf{1}(\sign(\hat{\eta}(\mathbf{x}))=-v_k) \right]f(\mathbf{x})d\mathbf{x}=\int_{B_k}\phi\left(\frac{\mathbf{x}-\mathbf{c}_k}{l}\right) f(\mathbf{x})d\mathbf{x}\nonumber\\
	&&\geq cl^d\int\phi(\mathbf{u})d\mathbf{u} = cl^d\norm{\phi}_1.
	\label{eq:lb2}
\end{eqnarray}
Therefore, if $\hat{v}_k\neq v_k$, then from \eqref{eq:lb1} and \eqref{eq:lb2},
\begin{eqnarray}
	\int_{B_k} \phi\left(\frac{\mathbf{x}-\mathbf{c}_k}{l}\right) \mathbf{1}(\sign(\hat{\eta}(\mathbf{x}))=-v_k) f(\mathbf{x})d\mathbf{x}\geq \frac{1}{2}cl^d \norm{\phi}_1.	
\end{eqnarray}
Denote the vector form $\hat{\mathbf{v}}=(\hat{v}_1, \ldots, \hat{v}_k)$. Then the Bayes risk is bounded by
\begin{eqnarray}
	R-R^*&=&\int|\eta_\mathbf{v}(\mathbf{x})|\text{P}(\sign(\hat{\eta}(\mathbf{x}))\neq \sign(\eta_\mathbf{v}(\mathbf{x})))f(\mathbf{x})d\mathbf{x}\nonumber\\
	&=&\sum_{k=1}^K \int_{B_k} |\eta_\mathbf{v}(\mathbf{x})|\text{P}(\sign(\hat{\eta}(\mathbf{x}))=-v_k) f(\mathbf{x})d\mathbf{x}\nonumber\\
	&=&l^\beta \sum_{k=1}^K \mathbb{E}\left[\int_{B_k} \phi\left(\frac{\mathbf{x}-\mathbf{c}_k}{l}\right) \mathbf{1}(\sign(\hat{\eta}(\mathbf{x})=-v_k))f(\mathbf{x})d\mathbf{x}\right]\nonumber\\
	&\geq & \frac{1}{2}cl^{\beta+d}\norm{\phi}_1 \mathbb{E}[\rho_H(\hat{\mathbf{v}}, \mathbf{v})],
\end{eqnarray}
in which $\rho_H(\hat{\mathbf{v}}, \mathbf{v})$ is the Hamming distance. Hence
\begin{eqnarray}
	\underset{\hat{Y}}{\inf}\underset{Q\in \mathcal{Q}_\epsilon}{\inf}\underset{(p,\eta)\in \mathcal{P}_{cls}}{\sup} (R-R^*)\geq \frac{1}{2} cl^{\beta+d}\norm{\phi}_1 \underset{\hat{\mathbf{v}}}{\inf}\underset{Q\in \mathcal{Q}_\epsilon}{\inf}\underset{\mathbf{v}\in \mathcal{V}}{\sup}\mathbb{E}[\rho_H(\hat{\mathbf{v}}, \mathbf{v})].
	\label{eq:transform}
\end{eqnarray}
Define
\begin{eqnarray}
	\delta = \underset{Q\in \mathcal{Q}_\epsilon}{\sup}\underset{\mathbf{v}, \mathbf{v}':\rho_H(\mathbf{v}, \mathbf{v}') = 1}{\max} D(p_{\mathbf{Z}|\mathbf{v}}||p_{\mathbf{Z}|\mathbf{v}'}),
	\label{eq:deltadef}
\end{eqnarray}
in which $p_{\mathbf{Z}|\mathbf{v}}$ denotes the distribution of privatized variable $\mathbf{Z}$ given $\eta=\eta_\mathbf{v}$. From \cite{tsybakov2009introduction}, Theorem 2.12(iv),
\begin{eqnarray}
	\underset{\hat{\mathbf{v}}}{\inf}\underset{\mathbf{v}\in \mathcal{V}}{\sup}\mathbb{E}[\rho_H(\hat{\mathbf{v}}, \mathbf{v})]\geq \frac{K}{2}\max\left(\frac{1}{2}e^{-\delta}, 1-\sqrt{\frac{\delta}{2}}\right).
\end{eqnarray}
It remains to bound $\delta$, i.e. $D(p_{\mathbf{Z}|\mathbf{v}}||p_{\mathbf{Z}|\mathbf{v}'})$ under the constraint that $\rho_H(\mathbf{v}, \mathbf{v}') = 1$. From \cite{duchi2018minimax}, Theorem 1, we have
\begin{eqnarray}
	D(p_{\mathbf{Z}|\mathbf{v}}||p_{\mathbf{Z}|\mathbf{v}'})\leq n(e^\epsilon - 1)^2 \mathbb{TV}^2(p_\mathbf{v}^m, p_{\mathbf{v}'}^m),
	\label{eq:TV}
\end{eqnarray}
in which $p_\mathbf{v}^m$ denotes the joint distribution of $(\mathbf{X}, Y)$ (i.e. before privatization) given $\eta=\eta_\mathbf{v}$. Note that $p_\mathbf{v}$ and $p_{\mathbf{v}'}$ are only different in one bin. Without loss of generality, suppose that $p_\mathbf{v}$ and $\mathbf{v}'$ are different at the first bin. Using Lemma \ref{lem:pinsker}, we have
\begin{eqnarray}
	&&\mathbb{TV}^2(p_\mathbf{v}^m, p_{\mathbf{v}'}^m)\nonumber\\
	&\overset{(a)}{\leq} & \frac{1}{2}p_\mathbf{v}(\mathbf{X}\in B_1)D(p_\mathbf{v}^m||p_{\mathbf{v}'}^m)\nonumber\\
	&=& \frac{1}{2}l^dD(p_\mathbf{v}^m||p_{\mathbf{v}'}^m)\nonumber\\
	&\leq & \frac{1}{2}ml^dD(p_\mathbf{v}||p_{\mathbf{v}'})\nonumber\\
	&=&\frac{1}{2}ml^d\int_{B_1} f(\mathbf{x})\left[p_\mathbf{v}(Y=1|\mathbf{x})\ln \frac{p_\mathbf{v}(Y=1|\mathbf{x})}{p_{\mathbf{v}'}(Y=1|\mathbf{x})}+p_\mathbf{v}(Y=-1|\mathbf{x})\ln \frac{p_\mathbf{v}(Y=-1|\mathbf{x})}{p_{\mathbf{v}'}(Y=-1|\mathbf{x})}\right]d\mathbf{x}\nonumber\\
	&\overset{(b)}{=}&\frac{1}{2}ml^d\int_{B_1} f(\mathbf{x})\left[\frac{1+\eta_\mathbf{v}(\mathbf{x})}{2}\ln \frac{1+\eta_\mathbf{v}(\mathbf{x})}{1-\eta_\mathbf{v}(\mathbf{x})}+\frac{1-\eta_\mathbf{v}(\mathbf{x})}{2}\ln \frac{1-\eta_\mathbf{v}(\mathbf{x})}{1+\eta_\mathbf{v}(\mathbf{x})}\right] d\mathbf{x}\nonumber\\
	&=&\frac{1}{2}ml^d\int_{B_1} f(\mathbf{x})\eta_\mathbf{v}(\mathbf{x})\ln \frac{1+\eta_\mathbf{v}(\mathbf{x})}{1-\eta_\mathbf{v}(\mathbf{x})}d\mathbf{x}\nonumber\\
	&\overset{(c)}{\leq} & \frac{3}{2}ml^d\int_{B_1} f(\mathbf{x})\eta_\mathbf{v}^2(\mathbf{x})d\mathbf{x}\nonumber\\
	&\leq &\frac{3}{2}ml^{d+2\beta}\int_{B_1}\phi^2\left(\frac{\mathbf{x}-\mathbf{c}_j}{h}\right) d\mathbf{x}\nonumber\\
	&=& \frac{3}{2}ml^{2d+2\beta} \norm{\phi}_2^2.
\end{eqnarray}
(a) holds because $p_\mathbf{v}$ and $p_{\mathbf{v}'}$ are only different at $B_1$. For (b), recall that $\eta(\mathbf{x})=\mathbb{E}[Y|\mathbf{X}=\mathbf{x}]$. (c) holds since $|\eta_\mathbf{v}(\mathbf{x})|\leq 1/2$ (recall the condition $\phi(\mathbf{u})l^\beta\leq 1/2$), if $v_1=1$, then $\ln (1+\eta_\mathbf{v}(\mathbf{x}))\leq \eta_\mathbf{v}(\mathbf{x})$, $\ln (1/(1-\eta_\mathbf{v}(\mathbf{x})))\leq 2\eta_\mathbf{v}(\mathbf{x})$. Similar result can be obtained for $v_1=-1$. From \eqref{eq:deltadef} and \eqref{eq:TV},
\begin{eqnarray}
	\delta\leq \frac{3}{2}n(e^\epsilon - 1)^2 ml^{2d+2\beta}\norm{\phi}_2^2.
\end{eqnarray}
Let 
\begin{eqnarray}
	l\sim (nm\epsilon^2)^{-\frac{1}{2(d+\beta)}},
\end{eqnarray}
then $\delta\lesssim 1$. Moreover, let $K\sim l^{\gamma \beta - d}$, then
\begin{eqnarray}
	\underset{\hat{\mathbf{v}}}{\inf}\underset{\mathbf{v}\in \mathcal{V}}{\sup}\mathbb{E}[\rho_H(\hat{\mathbf{v}}, \mathbf{v})]\gtrsim K\sim l^{\gamma \beta - d}.
\end{eqnarray}
From \eqref{eq:transform},
\begin{eqnarray}
	\underset{\hat{Y}}{\inf}\underset{Q\in \mathcal{Q}_\epsilon}{\inf}\underset{(p,\eta)\in \mathcal{P}_{cls}}{\sup} (R-R^*)\gtrsim l^{\beta+d} l^{\gamma \beta - d} = l^{\beta(1+\gamma)}\sim (nm\epsilon^2)^{-\frac{\beta(1+\gamma)}{2(d+\beta)}}.
\end{eqnarray}
\section{Nonparametric Regression}
\subsection{Algorithm Description}\label{sec:algreg}
Define $q_k$ and $Q_k$ in the same way as \eqref{eq:qk} and \eqref{eq:Qk}. Moreover, define
$p_k=\int_{B_k}f(\mathbf{x})d\mathbf{x}$, 
and
\begin{eqnarray}
	P_k=\int_{T_k} f(\mathbf{x})d\mathbf{x}-\int_{T_k^c} f(\mathbf{x})d\mathbf{x},
\end{eqnarray}
in which $T_k$ is defined in \eqref{eq:tk}, and $T_k^c$ is the complement. 

Denote
\begin{eqnarray}
	\eta_k:=\frac{q_k}{p_k}=\frac{\int_{B_k} f(\mathbf{x})\eta(\mathbf{x})d\mathbf{x}}{\int_{B_k} f(\mathbf{x})d\mathbf{x}},
	\label{eq:etak}
\end{eqnarray}
then $\eta_k$ can be viewed as the average of $\eta(\mathbf{x})$ weighted by the pdf. If $\eta$ is continuous and $l$ is sufficiently small, then $\eta(\mathbf{x})\approx \eta_k$ for all $\mathbf{x}\in B_k$. Hence, for any $\mathbf{x}\in B_k$, we can just estimate $\eta(\mathbf{x})$ by estimating $q_k$ and $p_k$. As has been discussed in the classification case, direct estimation is not efficient. Therefore, we estimate $\mathbf{Q}=(Q_1,\ldots, Q_K)$ and $\mathbf{P}=(P_1,\ldots, P_k)$ first, and then calculate $q_k$ and $p_k$ for $k=1,\ldots, K$.

\textbf{Training.} Recall that in the classification problem, we have divided the dataset into $K$ parts, which are used to estimate $Q_k$ for $k=1,\ldots, K$ respectively. For regression problem, we need to estimate both $Q_k$ and $P_k$. Therefore, now we divide the samples randomly into $2K$ groups, such that $K$ groups are used to estimate $Q_k$, $k=1,\ldots, K$, while the other $K$ groups are used to estimate $P_k$. The detailed steps are similar to the classification problem. In particular, $U_{ijk}$ is still calculated using \eqref{eq:U}. Since $\mathbb{E}[U_{ijk}]=Q_k$, $Q_k$ can still be estimated using \eqref{eq:Qkest}. To estimate $P_k$, let
\begin{eqnarray}
	V_{ijk}=\mathbf{1}(\mathbf{X}_{ij}\in T_k)-\mathbf{1}(\mathbf{X}_{ij}\in T_k^c).
	\label{eq:V}
\end{eqnarray}

Then we have $\mathbb{E}[V_{ijk}]=P_k$, and $|V_{ijk}|\leq 1$. Therefore, $P_k$ can be estimated similarly for $k=1,\ldots, K$:
\begin{eqnarray}
	\hat{P}_k=MeanEst1d(\{V_{ijk}|i\in S_{K+k}, j\in [m] \}).
	\label{eq:Pkest}
\end{eqnarray}

Note that samples are privatized in this step. With appropriate parameters, our method satisfies user-level $\epsilon$-LDP. Based on the values of $\hat{Q}_k$ and $\hat{P}_k$ for $k=1,\ldots, K$, $q_k$ and $p_k$ can be estimated by
\begin{eqnarray}
	\hat{\mathbf{q}}=\frac{1}{K}\mathbf{H}\hat{\mathbf{Q}},
	\hat{\mathbf{p}}=\frac{1}{K}\mathbf{H}\hat{\mathbf{P}},
\end{eqnarray}

in which $\hat{\mathbf{q}}=(\hat{q}_1,\ldots, \hat{q}_K)$, $\hat{\mathbf{p}}=(\hat{p}_1,\ldots, \hat{p}_K)$.

\textbf{Prediction.} For any test sample at $\mathbf{x}\in B_k$, The regression output is
\begin{eqnarray}
	\hat{\eta}(\mathbf{x})=\frac{\hat{q}_k}{\hat{p}_k}.
\end{eqnarray}

The whole training algorithm is summarized in Algorithm \ref{alg:reg}.
\begin{algorithm}[h!]
	\caption{Training algorithm of nonparametric regression under user-level $\epsilon$-LDP}\label{alg:reg}
	
	\textbf{Input:} Training dataset containing $n$ users with $m$ samples per user, i.e. $(\mathbf{X}_{ij}, Y_{ij})$, $i=1,\ldots, n$, $j=1,\ldots, m$\\
	\textbf{Output:} $\hat{\mathbf{q}}$, $\hat{\mathbf{p}}$\\
	\textbf{Parameter:} $h_q$, $h_p$, $\Delta_q$, $\Delta_p$, $l$
	\begin{algorithmic}
		\STATE Divide $\mathcal{X}=[0,1]^d$ into $B$ bins, such that the length of each bin is $l$;\\
		
		\STATE $K=2^{\lceil \log_2 B\rceil}$;\\
		
		\STATE Calculate $U_{ijk}$ according to \eqref{eq:U}, for $i=1,\ldots, n$, $j=1,\ldots, m$, $k=1,\ldots, K$; \\
		
		\STATE Calculate $V_{ijk}$ according to \eqref{eq:V}, for $i=1,\ldots, n$, $j=1,\ldots, m$, $k=1,\ldots, K$;
		
		\STATE Estimate $\hat{Q}_k$ using \eqref{eq:Qkest} with parameters $h_q$ and $\Delta_q$, for $k=1,\ldots, K$;\\
		
		\STATE Estimate $\hat{P}_k$ using \eqref{eq:Pkest} with parameters $h_p$ and $\Delta_p$, for $k=1,\ldots, K$;
		
		\STATE $\hat{\mathbf{q}}=\mathbf{H}_K\hat{\mathbf{Q}}/K$, in which $\hat{\mathbf{Q}}=(\hat{Q}_1,\ldots, \hat{Q}_K)^T$;\\
		
		\STATE $\hat{\mathbf{p}}=\mathbf{H}_K\mathbf{P}/K$, in which $\hat{\mathbf{P}}=(\hat{P}_1,\ldots, \hat{P}_K)^T$;
		
		\STATE \textbf{Return} $\hat{\mathbf{q}}$, $\hat{\mathbf{p}}$
	\end{algorithmic}		
\end{algorithm}
\subsection{Proof of Theorem \ref{thm:reg}}\label{sec:regub}
Define
\begin{eqnarray}
	\eta_k:=\frac{q_k}{p_k}
\end{eqnarray}
Recall the definition of $q_k$ and $p_k$, we have
\begin{eqnarray}
	\eta_k=\frac{\int_{B_k} f(\mathbf{x})\eta(\mathbf{x})d\mathbf{x}}{\int_{B_k}f(\mathbf{x})d\mathbf{x}},
\end{eqnarray}
and
\begin{eqnarray}
	\tilde{\eta}(\mathbf{x})=\sum_{k=1}^K \eta_k\mathbf{1}(\mathbf{x}\in B_k).
\end{eqnarray}
Then
\begin{eqnarray}
	R&=&\int(\hat{\eta}(\mathbf{x})-\eta(\mathbf{x}))^2 f(\mathbf{x})d\mathbf{x}\nonumber\\
	&\leq & 2\mathbb{E}\left[\int (\hat{\eta}(\mathbf{x})-\tilde{\eta}(\mathbf{x}))^2 f(\mathbf{x})d\mathbf{x} + 2\int (\tilde{\eta}(\mathbf{x})-\eta(\mathbf{x}))^2 f(\mathbf{x})d\mathbf{x}\right].
	\label{eq:Rdc}
\end{eqnarray}
The second term can be bounded with the following lemma.
\begin{lem}
	\begin{eqnarray}
		\int (\tilde{\eta}(\mathbf{x})-\eta(\mathbf{x}))^2 f(\mathbf{x})d\mathbf{x}\leq C_b^2 d^\beta l^{2\beta}.
	\end{eqnarray}
\end{lem}
\begin{proof}
	\begin{eqnarray}
		\int (\tilde{\eta}(\mathbf{x}) - \eta(\mathbf{x}))^2 f(\mathbf{x})d\mathbf{x}&=&\sum_{k=1}^K \int_{B_k} (\eta_k-\eta(\mathbf{x}))^2 f(\mathbf{x})d\mathbf{x}\nonumber\\
		&\leq & \sum_{k=1}^K \int_{B_k} \left(C_b(\sqrt{d}l)^\beta\right)^2 f(\mathbf{x})d\mathbf{x}\nonumber\\
		&=&C_b^2 d^\beta l^{2\beta}\sum_{k=1}^K p_k\nonumber\\
		&=& C_b d^\beta l^{2\beta}.
		\label{eq:Rb}
	\end{eqnarray}
\end{proof}
It remains to bound the first term of \eqref{eq:Rdc}.
\begin{lem}
	Denote $E_{1qk}$ and $E_{1pk}$ as the event that the first stage in estimating $Q_k$ and $P_K$ are successful, respectively. Denote $E_{1q}=\cap_k E_{1qk}$, $E_{1p}=\cap_k E_{1pk}$. Then there exists two constants $C_1$ and $C_2$, such that
	\begin{eqnarray}
		\text{P}(|\hat{q}_k-q_k|>t|E_{1q})\leq 2\exp\left[-C_1\frac{mn\epsilon^2}{T^2\ln n}\left(t-C_2\sqrt{\frac{T}{mn}}\right)^2\right],
	\end{eqnarray}
	and
	\begin{eqnarray}
		\text{P}(|\hat{p}_k-p_k|>t|E_{1p})\leq 2\exp\left[-C_1 \frac{mn\epsilon^2}{\ln n} \left(t-C_2\sqrt{\frac{1}{mn}}\right)^2\right].
	\end{eqnarray}
\end{lem}
Denote 
\begin{eqnarray}
	\hat{\eta}_k=\frac{\hat{q}_k}{\hat{p}_k}.
\end{eqnarray}
Pick some constant $c$ such that $C_1c^2>1$, then define
\begin{eqnarray}
	t_p&=&C_2\sqrt{\frac{1}{mn}}+\frac{c\ln n}{\sqrt{mn\epsilon^2}},\\
	t_q&=&C_2\sqrt{\frac{T}{mn}}+\frac{cT\ln n}{\sqrt{mn\epsilon^2}}.
\end{eqnarray}
Then
\begin{eqnarray}
	\text{P}(|\hat{p}_k-p_k|>t_p|E_{1p})\leq 2e^{-C_1c^2\ln n} = 2n^{-C_1c^2},
\end{eqnarray}
and
\begin{eqnarray}
	\text{P}(|\hat{q}_k-q_k|>t_q|E_{1q})\leq 2n^{-C_1c^2}.
\end{eqnarray}
Denote $E$ as the event that for all $k$, $|\hat{p}_k-p_k|>t_p$, $\hat{q}_k-q_k|>t_q$. Then
\begin{eqnarray}
	P(E^c)&=&\text{P}(\exists k, |\hat{p}_k-p_k|>t_p \text{ or } |\hat{q}_k-q_k|>t_q)\nonumber\\
	&\leq& 4Bn^{-C_1c^2}+\text{P}\left(E_{p1}^c \cup E_{q1}^c\right)\nonumber\\
	&\leq & 4B\left(n^{-C_1c^2}+\sqrt{m} e^{-C_0n\epsilon^2}\right).
\end{eqnarray}
Hence
\begin{eqnarray}
	\mathbb{E}\left[\int (\hat{\eta}(\mathbf{x}) - \tilde{\eta}(\mathbf{x}))^2 f(\mathbf{x})d\mathbf{x} \mathbf{1}(E^c)\right] \leq T^2 \text{P}(E^c)\leq 4BT^2\left(n^{-C_1c^2}+\sqrt{m} e^{-C_0n\epsilon^2}\right).
	\label{eq:e}
\end{eqnarray}
With $C_1c^2\geq 1$, this term does not dominate.

Under $E$, we have
\begin{eqnarray}
	\int (\hat{\eta}(\mathbf{x}) - \tilde{\eta}(\mathbf{x}))^2 f(\mathbf{x})d\mathbf{x} &=& \sum_{k=1}^K \int_{B_k} (\hat{\eta}(\mathbf{x}) - \eta_k)^2 f(\mathbf{x})d\mathbf{x}\nonumber\\
	&=&\sum_{k=1}^K (\hat{\eta}_k-\eta_k)^2 \int_{B_k} f(\mathbf{x})d\mathbf{x}\nonumber\\
	&=&\sum_{k=1}^K p_k(\hat{\eta}_k-\eta_k)^2.
\end{eqnarray}
$\hat{\eta}_k-\eta_k$ can be bounded in both two sides:
\begin{eqnarray}
	\hat{\eta}_k-\eta_k &=& \frac{\hat{q}_k}{\hat{p}_k}\wedge T-\frac{q_k}{p_k}\nonumber\\
	&\leq & \frac{q_k+t_q}{p_k-t_p} - \frac{q_k}{p_k} = \frac{p_kt_q+q_kt_p}{p_k(p_k-t_p)},
\end{eqnarray}
and
\begin{eqnarray}
	\hat{\eta}_k-\eta_k&\geq& \frac{q_k-t_q}{p_k+t_p}-\frac{q_k}{p_k}\nonumber\\
	&=&-\frac{p_kt_q+q_kt_p}{p_k(p_k+t_p)}.
\end{eqnarray}
Note that $f(\mathbf{x})\geq f_L$, thus $p_k\geq f_Ll^d$. Ensure that $l$ is picked such that $f_Ll^d\geq 2t_p$. Then
\begin{eqnarray}
	|\hat{\eta}_k-\eta_k|\leq 2\frac{p_kt_q+q_kt_p}{p_k^2},
\end{eqnarray}
\begin{eqnarray}
	(\hat{\eta}_k-\eta_k)^2\leq 8 \left(\frac{t_q^2}{p_k^2}+\frac{q_k^2t_p^2}{p_k^4}\right).
\end{eqnarray}
Hence
\begin{eqnarray}
	\mathbb{E}\left[\int (\hat{\eta}(\mathbf{x}) - \tilde{\eta}(\mathbf{x}))^2 f(\mathbf{x})d\mathbf{x} \mathbf{1}(E)\right]&\leq& \sum_{k=1}^K p_k\left(\frac{t_q^2}{p_k^2} + \frac{q_k^2 t_p^2}{p_k^4}\right)\nonumber\\
	&\lesssim & \frac{\ln^2 n}{mn\epsilon^2 l^{2d}}.
	\label{eq:ec}
\end{eqnarray}
From \eqref{eq:e} and \eqref{eq:ec},
\begin{eqnarray}
	\mathbb{E}\left[(\hat{\eta}(\mathbf{x}) - \tilde{\eta}(\mathbf{x}))^2 f(\mathbf{x})d\mathbf{x}\right]\lesssim \frac{\ln^2 n}{mn\epsilon^2 l^{2d}}.
	\label{eq:Rv}
\end{eqnarray}
From \eqref{eq:Rdc}, \eqref{eq:Rb} and \eqref{eq:Rv}, 
\begin{eqnarray}
	R\lesssim \frac{\ln^2 n}{mn\epsilon^2 l^{2d}}+l^{2\beta}.
\end{eqnarray}
Let 
\begin{eqnarray}
	l \sim \left(\frac{mn\epsilon^2}{\ln^2 n}\right)^{-\frac{1}{2(d+\beta)}},
\end{eqnarray}
then
\begin{eqnarray}
	R\lesssim \left(\frac{mn\epsilon^2}{\ln^2 n}\right)^{-\frac{\beta}{d+\beta}}.
\end{eqnarray}
\subsection{Proof of Theorem \ref{thm:reglb}}\label{sec:reglb}
Similar to the classification case, divide support $\mathcal{X}=[0,1]^d$ into $B$ bins with length $l$, then $Bl^d=1$. Let $\phi(\mathbf{u})$ be some function supported at $[-1/2, 1/2]^d$, $\phi(\mathbf{u})\geq 0$, and for any $\mathbf{u}, \mathbf{u}'$,
\begin{eqnarray}
	|\phi(\mathbf{u})-\phi(\mathbf{u}')|\leq C_b\norm{\mathbf{u}-\mathbf{u}'}_2^\beta.
\end{eqnarray}
Suppose $\mathbf{c}_1,\ldots, \mathbf{c}_B$ be the centers of $B$ bins, $f(\mathbf{x}) = 1$, and
\begin{eqnarray}
	\eta(\mathbf{x})=\sum_{k=1}^B v_k\phi\left(\frac{\mathbf{x}-\mathbf{c}_k}{l}\right) l^\beta,
\end{eqnarray}
in which $v_k\in \{-1,1\}$. Then let 
\begin{eqnarray}
	\hat{v}_k=\underset{s\in \{-1,1\}}{\arg\min}\int_{B_k} \left(\hat{\eta}(\mathbf{x}) - s\phi\left(\frac{\mathbf{x}-\mathbf{c}_k}{l}\right)l^\beta\right)^2f(\mathbf{x})d\mathbf{x}.
\end{eqnarray}
Then
\begin{eqnarray}
	R&=&\mathbb{E}\left[\int (\hat{\eta}(\mathbf{x}) - \eta(\mathbf{x}))^2 f(\mathbf{x})d\mathbf{x}\right]\nonumber\\
	&=&\sum_{k=1}^B \mathbb{E}\left[\int_{B_k} \left(\hat{\eta}(\mathbf{x}) - \eta(\mathbf{x})\right)^2 f(\mathbf{x})d\mathbf{x}\right]\nonumber\\
	&\overset{(a)}{\geq} & \sum_{k=1}^B l^{2\beta+d}\norm{\phi}_2^2\text{P}(\hat{v}_k\neq v_k)\nonumber\\
	&=&\norm{\phi}_2^2 l^{2\beta+d}\mathbb{E}[\rho_H(\hat{\mathbf{v}}, \mathbf{v})].
	\label{eq:Rreg}
\end{eqnarray}
Here we explain (a). Without loss of generality, suppose $v_k=-1$, $\hat{v}_k=1$. Then
\begin{eqnarray}
	\int_{B_k} \left(\hat{\eta}(\mathbf{x}) - \phi\left(\frac{\mathbf{x}-\mathbf{c}_k}{l}\right)l^\beta\right)^2 f(\mathbf{x})d\mathbf{x}\leq\int_{B_k} \left(\hat{\eta}(\mathbf{x}) + \phi\left(\frac{\mathbf{x}-\mathbf{c}_k}{l}\right)l^\beta\right)^2 f(\mathbf{x})d\mathbf{x}.  
\end{eqnarray}
Note that
\begin{eqnarray}
	&&\int_{B_k}\left(\hat{\eta}(\mathbf{x}) - \phi\left(\frac{\mathbf{x}-\mathbf{c}_k}{l}\right) l^\beta\right)^2 f(\mathbf{x})d\mathbf{x} + \int_{B_k}\left(\hat{\eta}(\mathbf{x}) + \phi\left(\frac{\mathbf{x}-\mathbf{c}_k}{l}\right) l^\beta\right)^2 f(\mathbf{x})d\mathbf{x}\nonumber\\
	&=&2 \int_{B_k} \left(\hat{\eta}^2(\mathbf{x}) + \phi^2 \left(\frac{\mathbf{x}-\mathbf{c}_k}{l}\right)l^{2\beta}\right) f(\mathbf{x})d\mathbf{x}\nonumber\\
	&\geq & 2l^{2\beta + d}\norm{\phi}_2^2.
\end{eqnarray}
Thus
\begin{eqnarray}
	\int_{B_k} \left(\hat{\eta}(\mathbf{x}) + \phi\left(\frac{\mathbf{x}-\mathbf{c}_k}{l}\right) l^\beta\right)^2 f(\mathbf{x})d\mathbf{x}\geq l^{2\beta + d}\norm{\phi}_2^2.
\end{eqnarray}
Similar bound holds if $v_k=1$ and $\hat{v}_k=-1$. Now (a) in \eqref{eq:Rreg} has been proved. From \eqref{eq:Rreg},
\begin{eqnarray}
	\underset{\hat{\eta}}{\inf}\underset{(f, \eta)\in \mathcal{P}_{reg}}{\sup} R\geq \norm{\phi}_2^2 l^{2\beta+d}\underset{\hat{\mathbf{v}}}{\inf}\underset{\mathbf{v}}{\sup}\mathbb{E}[\rho_H(\hat{\mathbf{v}}, \mathbf{v})].
	\label{eq:hamming}
\end{eqnarray}
Define 
\begin{eqnarray}
	\delta = \underset{\mathbf{v}, \mathbf{v}': \rho_H(\mathbf{v}, \mathbf{v}') = 1}{\max} D(p_{\mathbf{Z}|\mathbf{v}}||p_{\mathbf{Z}|\mathbf{v}'}).
\end{eqnarray}
Follow the analysis of nonparametric classification, let
\begin{eqnarray}
	l\sim (nm\epsilon^2)^{-\frac{1}{2(d+\beta)}},
\end{eqnarray}
then $\delta\lesssim 1$. Hence, By \cite{tsybakov2009introduction}, Theorem 2.12(iv), 
\begin{eqnarray}
	\underset{\hat{\mathbf{v}}}{\inf}\underset{\mathbf{v}}{\sup}\mathbb{E}[\rho_H(\hat{\mathbf{v}}, \mathbf{v})]\gtrsim B\sim l^{-d},
\end{eqnarray}
hence from \eqref{eq:hamming},
\begin{eqnarray}
	\underset{\hat{\eta}}{\inf}\underset{(f, \eta)\in \mathcal{P}_{reg}}{\sup} R\gtrsim l^{2\beta+d} \cdot l^{-d} = h^{2\beta}\sim (nm\epsilon^2)^{-\frac{\beta}{d+\beta}}.
\end{eqnarray}
\section{Auxiliary Lemmas}
\begin{lem}\label{lem:pinsker}
	Suppose there are two probability measures $p_1$ and $p_2$ supported at $\mathcal{X}$. $p_1=p_2$ except at $S\subset \mathcal{X}$. Then
	\begin{eqnarray}
		\mathbb{TV}(p_1,p_2)\leq \sqrt{\frac{1}{2}p_1(S)D(p_1||p_2)}.
	\end{eqnarray}
\end{lem}
\begin{proof}
	Denote $\mathbb{E}_1$ as the expectation under $p_1$. Denote $p_{1|S}$ and $p_{2|S}$ as the conditional distribution of $p_1$ and $p_2$ on $S$.
	\begin{eqnarray}
		D(p_1||p_2)&=&\mathbb{E}_1\left[\ln \frac{p_1}{p_2}\right]\nonumber\\
		&=&p_1(S)\mathbb{E}_{1|S}\left[\ln \frac{p_{1|S}}{p_{2|S}}\right]\nonumber\\
		&\geq & 2p_1(S)\mathbb{TV}^2(p_{1|S}, p_{2|S})\nonumber\\
		&=&2p_1(S)\left[\frac{\mathbb{TV}(p_1,p_2)}{p_1(S)}\right]^2\nonumber\\
		&=&\frac{2\mathbb{TV}^2(p_1,p_2)}{p_1(S)}.
	\end{eqnarray}
	The proof is complete.
\end{proof}
\begin{lem}\label{lem:tail}
	Under Assumption \ref{ass:cls}(a), there exists a constant $C$, such that for any $s>0$, 
	\begin{eqnarray}
		\mathbb{E}\left[|\eta(\mathbf{X})|e^{-s|\eta(\mathbf{X})|^2}\right]\leq Cs^{-\frac{1}{2}(\gamma+1)}.
	\end{eqnarray}
\end{lem}
\begin{proof}
	\begin{eqnarray}
		\mathbb{E}\left[|\eta(\mathbf{X})|e^{-s|\eta(\mathbf{X})|^2}\right]&=&\mathbb{E}\left[|\eta(\mathbf{X})|e^{-\frac{s}{2}|\eta(\mathbf{X})|^2}e^{-\frac{s}{2}|\eta(\mathbf{X})|^2}\right]\nonumber\\
		&\leq & \left(\underset{u\geq0}{\sup} ue^{-\frac{s}{2}u^2}\right) \mathbb{E}\left[e^{-\frac{s}{2}|\eta(\mathbf{X})|^2}\right]\nonumber\\
		&=&\frac{1}{\sqrt{s}}e^{-\frac{1}{2}}\mathbb{E}\left[e^{-\frac{s}{2}|\eta(\mathbf{X})|^2}\right]\nonumber\\
		&=&\frac{1}{\sqrt{s}}e^{-\frac{1}{2}}\int_0^1 \text{P}\left(e^{-\frac{s}{2}|\eta(\mathbf{X})|^2}> t\right) dt\nonumber\\
		&=&\frac{1}{\sqrt{s}}e^{-\frac{1}{2}}\int_0^1 \text{P}\left(|\eta(\mathbf{X})|<\sqrt{\frac{2\ln \frac{1}{t}}{s}}\right) dt\nonumber\\
		&\leq & \frac{C_a}{\sqrt{s}} e^{-\frac{1}{2}} \int_0^1 \left(\frac{2\ln \frac{1}{t}}{s}\right)^\frac{\gamma}{2} dt\nonumber\\
		&\leq & 2^\frac{\gamma}{2}C_ae^{-\frac{1}{2}}s^{-\frac{1+\gamma}{2}}\Gamma\left(\frac{\gamma}{2}+1\right),
	\end{eqnarray}
	in which $\Gamma(u)=\int_0^\infty t^{u-1} e^{-t} dt$ is the Gamma function.
\end{proof}

\end{document}